\documentclass[aos,preprint]{imsart}
\RequirePackage[OT1]{fontenc}
\RequirePackage{amsthm,amsmath}
\RequirePackage[numbers]{natbib}
\RequirePackage[colorlinks,citecolor=blue,urlcolor=blue]{hyperref}
\startlocaldefs

\usepackage{wasysym,amssymb}
\usepackage{mathrsfs,bbm,euscript,dsfont,enumitem}
\setlist[description]{font=\normalfont\space}
\DeclareMathAlphabet{\pazocal}{OMS}{zplm}{m}{n}
\def\rn{\mathbb{R}}
\def\cn{\mathbb{C}}
\def\nn{\mathbb{N}}

\def\l{\left}
\def\r{\right}

\def\sF{\pazocal{F}}
\def\sG{\pazocal{G}}
\def\sM{\pazocal{M}}
\def\sX{\pazocal{X}}
\def\sD{\pazocal{D}}

\def\sP{\mathscr{P}}
\def\sM{\mathscr{M}}
\def\sQ{\mathscr{Q}}
\def\sL{\pazocal{L}}
\def\bX{\mathbf{X}}

\def\hs{\mathscr{HS}}
\def\pr{\mathbb{P}}

\def\d2{\sD_2}
\def\dd{\Delta\left( \sD \right)}

\def\simiid{\overset{iid}{\sim}}
\def\spn{\operatorname{span}}

\def\cip{\overset{p}{\rightarrow}}
\def\lrd{{\ell^2\left( \rn^d \right)}}

\theoremstyle{plain}
\newtheorem{thm}{Theorem}[section]
\newtheorem{lem}{Lemma}[section]
\newtheorem{cor}{Corollary}[section]
\theoremstyle{defintion}
\newtheorem{defin}{Definition}[section]

\endlocaldefs

\begin{document}

\begin{frontmatter}
	\title{An Operator Theoretic Approach to Nonparametric Mixture Models}
	\runtitle{Operator Theory for Mixture Models}

	\begin{aug}
		\author{\fnms{Robert} \snm{Vandermeulen}\corref{}\thanksref{rob}\ead[label=e1]{rvdm@umich.edu}}
		\and
		\author{\fnms{Clayton} \snm{Scott}\thanksref{clay} \ead[label=e2]{clayscot@umich.edu}}

		\runauthor{Vandermeulen and Scott}

		\affiliation{University of Michigan: Electrical and Computer Engineering\thanksmark{rob}\thanksmark{clay}, Statistics\thanksmark{clay}}

		\address{EECS Building\\
			1301 Beal Avenue\\
			Ann Arbor MI, 48109\\
			\printead{e1}\\
		\phantom{E-mail:\ }\printead*{e2}}
	\end{aug}

	\begin{abstract}
		When estimating finite mixture models, it is common to make assumptions on the mixture components, such as parametric assumptions. In this work, we make no distributional assumptions on the mixture components and instead assume that observations from the mixture model are grouped, such that observations in the same group are known to be drawn from the same mixture component. We precisely characterize the number of observations $n$ per group needed for the mixture model to be identifiable, as a function of the number $m$ of mixture components. In addition to our assumption-free analysis, we also study the settings where the mixture components are either linearly independent or jointly irreducible. Furthermore, our analysis considers two kinds of identifiability -- where the mixture model is the simplest one explaining the data, and where it is the only one. As an application of these results, we precisely characterize identifiability of multinomial mixture models. Our analysis relies on an operator-theoretic framework that associates mixture models in the grouped-sample setting with certain infinite-dimensional tensors. Based on this framework, we introduce general spectral algorithms for recovering the mixture components and illustrate their use on a synthetic data set.
	\end{abstract}



\end{frontmatter}

\section{Introduction}
A finite mixture model $\sP$ is a probability measure over a space of probability measures where $\sP\left( \left\{ \mu_i \right\} \right)=w_i >0$ for some finite collection of probability measures $\mu_1,\ldots,\mu_m$ and $\sum_{i=1}^m w_i = 1$. A realization from this mixture model first randomly selects some mixture component $\mu \sim \sP$ and then draws from $\mu$. Mixture models have seen extensive use in statistics and machine learning.

A central theoretical question concerning mixture models is that of identifiability. A mixture model is said to be {\em identifiable} if there is no other mixture model that defines the same distribution over the observed data. Classically mixture models were concerned with the case where the observed data $X_1,X_2,\ldots$ are iid with $X_i$ distributed according to some unobserved random measure $\mu_i$ with $\mu_i\simiid \sP$. This situation is equivalent to $X_i \simiid \sum_{j=1}^m w_j \mu_j$. If we impose no restrictions on the mixture components $\mu_1,\ldots,\mu_m$ one could easily concoct many choices of $\mu_j$ and $w_j$ which yield an identical distribution on $X_i$. Because of this, most previous work on identifiability assumes some sort of structure on $\mu_1,\ldots,\mu_m$, such as Gaussianity \cite{anderson14,bruni85, yakowitz68}. In this work we consider an alternative scenario where we make no assumptions on $\mu_1,\ldots,\mu_m$ and instead have access to groups of samples that are known to come from the same component. We will call these groups of samples ``random groups.'' Mathematically a random group is a random element $\bX_i$ where $\bX_i = \left( X_{i,1},\ldots,X_{i,n} \right)$ with $X_{i,1},\ldots,X_{i,n}\simiid \mu_i$ and $\mu_i \simiid \sP$.

In this paper we show that every mixture model with $m$ components is $(2m - 1)$-identifiable and $2m$-determined. Furthermore we show that any mixture model with linearly independent components is $3$-identifiable and $4$-determined, and any mixture model with jointly irreducible components is $2$-determined. These results, presented in Section \ref{sec:mainresults}, hold for any mixture model over any space and cannot be improved. The operator theoretic framework underlying our analysis is presented in Section \ref{sec:tensorproducts}, and the proofs our our main results appear in Section \ref{sec:proofsoftheorems}. In Section \ref{sec:multinomial}, we apply our main results to demonstrate some new and old results on the identifiability of multinomial mixture models. Section \ref{sec:algorithm} describes and analyzes a spectral algorithm for the recovery of the mixture components and weights, and experimental results on simulated data are presented in Section \ref{sec:experiments}. Related work, the problem formulation, and a concluding discussion are offered in Sections \ref{sec:previouswork}, \ref{sec:problemsetup}, and \ref{sec:discussion}, respectively.

	This paper contains and greatly expands on the results in our other work \cite{arxiv15}. Consequently there is some amount of overlap with this paper and \cite{arxiv15}.
        \section{Previous Work}\label{sec:previouswork}
	In classical mixture model theory identifiability is achieved by making assumptions about the mixture components. Some assumptions which yield identifiability are Gaussian or binomial mixture components \cite{bruni85,teicher63}. If one makes no assumptions on the mixture components then one must leverage some other type of structure in order to achieve identifiability. An example of such structure exists in the context of multiview models. In a multiview model samples have the form $\bX_i = \left( X_{i,1},\ldots, X_{i,n} \right)$ and the distribution of $\bX_i$ is defined by $\sum_{i=1}^m w_i \prod_{j=1}^n \mu_i^j$. In \cite{allman09} it was shown that if $\mu_i^j$ are probability distributions on $\rn$ with $\mu_1^j,\ldots,\mu_m^j$ linearly independent for all $j$ and $n\ge 3$, then the model is identifiable.

	The setting which we investigate is a special case of the multiview model where $\mu_i^j = \mu_i^{j'}$ for all $i,j,j'$. If the sample space of the $\mu_i$ is finite then this problem is exactly the topic modelling problem with a finite number of topics and one topic for each document. In topic modelling each $\mu_i$ is a ``topic'' and the sample space is a finite collection of words. This setting is well studied and it has been shown that one can recover the true topics provided certain assumptions on the topics are satisfied \cite{allman09, anandkumar14, arora12}. This problem was studied for arbitrary topics in \cite{rabani14}. In this paper the authors introduce an algorithm that recovers any mixture of $m$ topics provided $2m-1$ words per document. They also show, in a result analogous to our own, that this $2m-1$ value cannot be improved. Our proof techniques are quite different than those used in \cite{rabani14}, hold for arbitrary sample spaces, and are less complex. In Lemma \ref{lem:mult} we show that, when restricted to categorical spaces, the grouped sample setting introduced in this paper is equivalent to a multinomial mixture model. Fundamental bounds on the identifiability of multinomial mixture models can be found in \cite{kim1984,elmore2003}. We will reproduce these results (and develop some new results) using techniques developed in this paper. Additional connections to previous work are given later.
        \section{Problem Setup} \label{sec:problemsetup}
	We treat this problem in a general setting. For any measurable space we define $\delta_x$ as the Dirac measure at $x$. For $\Upsilon$ a set, $\sigma$-algebra, or measure, we denote $\Upsilon^{\times a}$ to be the standard $a$-fold product associated with that object. Let $\nn$ be the set of integers greater than or equal to zero and $\nn_+$ be the integers strictly greater than 0. For $k \in \nn_+$, we define $\left[ k \right] \triangleq \mathbb{N_+} \cap \left[ 1,k \right]$.
	Let $\Omega$ be a set containing more than one element. This set is the sample space of our data. Let $\sF$ be a $\sigma$-algebra over $\Omega$. Assume $\sF \neq \left\{ \emptyset, \Omega \right\}$, i.e. $\sF$ contains nontrivial events. We denote the space of probability measures over this space as $\sD\left( \Omega,\sF \right)$, which we will shorten to $\sD$. We will equip $\sD$ with the $\sigma$-algebra $2^\sD$ so that each Dirac measure over $\sD$ is unique. Define $\dd \triangleq \spn \left( \l\{\delta_x: x \in \sD \r\}\right)$. This is the ambient space where our mixtures of probability measures live. Let $\sP = \sum_{i=1}^m  w_i \delta_{\mu_i}$ be a probability measure in $\dd$. Let $\mu\sim \sP$ and $X_1 ,\ldots, X_n \simiid \mu$. Here $\bX$ is a random group sample, which was described in the introduction. We will denote $\bX = \left( X_1,\ldots,X_n \right)$.

	We now derive the probability law of $\bX$. Let $A\in \sF^{\times n}$. Letting $\pr$ reflect both the draw of $\mu\sim \sP$ and $X_1,\ldots,X_n \simiid \mu$, we have
	\begin{eqnarray*}
		\pr\left(\bX \in A \right)
		&=& \sum_{i=1}^m \pr\left( \bX \in A \right|\mu=\mu_i) \pr\left( \mu=\mu_i \right)\\
	 &=&  \sum_{i=1}^m w_i \mu_i^{\times n}\left( A \right).
	\end{eqnarray*}
	The second equality follows from Lemma 3.10 in \cite{fomp}.
	So the probability law of $\bX$ is 
	\begin{eqnarray*}
		\sum_{i=1}^m w_i \mu_i^{\times n}. 
	\end{eqnarray*}
	We want to view the probability law of $\bX$ as a function of $\sP$ in a mathematically rigorous way, which requires a bit of technical buildup.
	Let $\sQ\in \dd$. From the definition of $\dd$ it follows that $\sQ$ admits the representation $$\sQ = \sum_{i=1}^r \alpha_i\delta_{\nu_i} .$$
	From the well-ordering principle there must exist some representation with minimal $r$ and we define this $r$ as the {\it order} of $\sQ$. We can show that the minimal representation of any $\sQ \in \dd$ is unique up to permutation of its indices.

	\begin{lem} \label{lem:represent}
		Let $\sQ\in \dd$ and admit minimal representations $\sQ = \sum_{i=1}^r  \alpha_i \delta_{\nu_i}= \sum_{j=1}^r \alpha_j'\delta_{\nu_j'}$. There exists some permutation $\psi:\left[ r \right] \to \left[ r \right]$ such that $\nu_{\psi\left( i \right)} = \nu'_i$ and $\alpha_{\psi\left( i \right)} = \alpha'_i$ for all $i$.
	\end{lem}
	Henceforth when we define an element of $\dd$ with a summation we will assume that the summation is a minimal representation.
	\begin{defin} \label{def:mixmeasure}
		We call $\sP =\sum_{i=1}^m w_i \delta_{\mu_i}$ a {\em mixture of measures} if it is a probability measure in $\dd$. The elements $\mu_1,\ldots,\mu_m$, are called {\em mixture components}.
	\end{defin}

	Any minimal representation of a mixture of measures $\sP$ with $m$ components satisfies $\sP=\sum_{i=1}^m w_i \delta_{\mu_i}$ with $w_i>0$ for all $i$ and $\sum_{i=1}^m w_i = 1$. Hence any mixture of measures is a convex combination of Dirac measures at elements in $\sD$.

	For a measurable space $\left(\Psi, \sG \right)$ we define $\sM \left(\Psi, \sG \right)$ as the space of all finite signed measures over $\left(\Psi, \sG \right)$. We can now introduce the operator $V_n:\dd\to \sM\left( \Omega^{\times n}, \sF^{\times n} \right)$. For a minimal representation $\sQ =\sum_{i=1}^r  \alpha_i\delta_{\nu_i}$, we define $V_n$, with $n \in \nn_+$, as 
	\begin{eqnarray*}
		V_n(\sQ) =\sum_{i=1}^r  \alpha_i\nu_i^{\times n}.
	\end{eqnarray*}
	This mapping is well defined as a consequence of Lemma \ref{lem:represent}.
	From this definition we have that $V_n\left( \sP \right)$ is simply the law of $\bX$ which we derived earlier. In the following definitions, two mixtures of measures are considered equal if they define the same measure.

	\begin{defin}\label{def:ident}
		We call a mixture of measures, $\sP$, \emph{$n$-identifiable} if there does not exist a different mixture of measures $\sP'$, with order no greater than the order of $\sP$, such that $V_n\left( \sP \right) = V_n\left( \sP' \right)$.
	\end{defin}
	\begin{defin}\label{def:det}
		We call a mixture of measures, $\sP$, \emph{$n$-determined} if there exists no other mixture of measures $\sP'$ such that $V_n\left( \sP \right) = V_n\left( \sP' \right)$. 
	\end{defin}

	Definition \ref{def:ident} and \ref{def:det} are central objects of interest in this paper. Given a mixture of measures, $\sP = \sum_{i=1}^m w_i\delta_{\mu_i}$ then $V_n(\sP)$ is equal to $\sum_{i=1}^m w_i \mu_i^{\times n}$, the measure from which $\bX$ is drawn. If $\sP$ is not $n$-identifiable then we know that there exists a different mixture of measures that is no more complex (in terms of number of mixture components) than $\sP$ which induces the same distribution on $\bX$. Practically speaking this means we need more samples in each random group $\bX$ in order for the full richness of $\sP$ to be manifested in $\bX$. A stronger version of $n$-identifiability is $n$-determinedness where we enforce the requirement that our mixture of measures be the {\em only} mixture of measures (of any order) that admits the distribution on $\bX$.
	
	A quick note on terminology. We use the term ``mixture of measures'' rather than ``mixture model'' to emphasize that a mixture of measures should be interpreted a bit differently than a typical mixture model. A ``mixture model'' connotes a probability measure on the sample space of observed data $\Omega$, whereas a ``mixture of measures'' connotes a probability measure on the sample space of the unobserved latent measures $\sD$.

        \section{Main Results} \label{sec:mainresults}
	The first result is a bound on the $n$-identifiability of all mixtures of measures with $m$ or fewer components. This bound cannot be uniformly improved.
	\begin{thm} \label{thm:ident}
		Let $\left( \Omega,\sF \right)$ be a measurable space. Mixtures of measures with $m$ components are $(2m-1)$-identifiable.
	\end{thm}

	\begin{thm} \label{thm:noident}
		Let $\left( \Omega, \sF \right)$ be a measurable space with $\sF \neq \left\{ \emptyset,\Omega \right\}$. For all $m\ge 2$, there exists a mixture of measures with $m$ components that is not $(2m-2)$-identifiable.
	\end{thm}
	The following lemmas convey the unsurprising fact that $n$-identifiability is, in some sense, monotonic.
	\begin{lem}\label{lem:ident} \sloppy
		If a mixture of measures is $n$-identifiable then it is $q$-identifiable for all $q>n$.
	\end{lem}
	\begin{lem} \label{lem:noident}
		If a mixture of measures is not $n$-identifiable then it is not $q$-identifiable for any $q<n$.
	\end{lem}
	Viewed alternatively these results say that $n=2m-1$ is the smallest value for which $V_{n}$ is injective over the set of mixtures of measures with $m$ or fewer components.

	We also present an analogous bound for $n$-determinedness. This bound also cannot be improved.
	\begin{thm}\label{thm:det}
		Let $\left( \Omega,\sF \right)$ be a measurable space. Mixtures of measures with $m$ components are $2m$-determined.
	\end{thm}

	\begin{thm} \label{thm:nodet}
		Let $\left( \Omega, \sF \right)$ be a measurable space with $\sF \neq \left\{ \emptyset,\Omega \right\}$. For all $m$, there exists a mixture of measures with $m$ components that is not $(2m-1)$-determined.
	\end{thm}

	Again $n$-determinedness is monotonic in the number of samples per group.

	\begin{lem}\label{lem:det}
		If a mixture of measures is $n$-determined then it is $q$-determined for all $q>n$. 
	\end{lem}
	\begin{lem} \label{lem:nodet}
		If a mixture of measures is not $n$-determined then it is not $q$-determined for any $q<n$. 
	\end{lem}
	This collection of results can be interpreted in an alternative way. Consider some pair of mixtures of measures $\sP, \sP'$. If $n\ge2m$ and either mixture of measures is of order $m$ or less, then $V_n\left( \sP \right) = V_n\left( \sP' \right)$ implies $\sP = \sP'$. Furthermore $n=2m$ is the smallest value of $n$ for which the previous statement is true for all pairs of mixtures of measures.

	Our definitions of $n$-identifiability, $n$-determinedness, and their relation to previous works on identifiability deserve a bit of discussion. Some previous works on identifiability contain results related to what we call ``identifiability'' and others contain results related what we call ``determinedness.'' Both of these are simply called ``identifiability'' in these works. For example in \cite{yakowitz68} it is shown that different finite mixtures of multivariate Gaussian distributions will always yield different distributions, a result which we could call ``determinedness.'' Alternatively, in \cite{teicher63} it is demonstrated that mixtures of binomial distributions, with a fixed number of trials $n$ for every mixture component, are identifiable provided we only consider mixtures with $m$ mixture components and $n \ge 2m-1$. In this result allowing for more mixture components may destroy identifiability and thus this is what {\em we} would call an ``identifiability'' result. The fact that the value $2m-1$ occurs in both the previous binomial mixture model result and Theorem \ref{thm:ident} is not a coincidence. We will demonstrate a new determinedness result for multinomial mixtures models later in the paper, under the assumption that $n\ge 2m$. We will prove these results using Theorems \ref{thm:ident} and \ref{thm:det}. To our knowledge our work is the first to consider both identifiability and determinedness.

	Finally we also include results that are analogous to previously shown results for the discrete setting. We note that our proof techniques are markedly different than the previous proofs for the discrete case.
	\begin{thm} \label{thm:liident}
		If $\sP = \sum_{i=1}^m  w_i\delta_{\mu_i}$ is a mixture of measures where $\mu_1,\ldots,\mu_m$ are linearly independent then $\sP$ is $3$-identifiable.
	\end{thm}
	This bound is tight as a consequence of Theorem \ref{thm:noident} with $m=2$ since any pair of distinct measures must be linearly independent.

	A version of this theorem was first proven in \cite{allman09} by making use of Kruskal's Theorem \cite{kruskal77}. Kruskal's Theorem demonstrates that order 3 tensors over $\rn^d$ admit unique decompositions (up to scaling and permutation) given certain linear independence assumptions. Our proof makes no use of Kruskal's Theorem and demonstrates that $n$-identifiability for linearly independent mixture components need not be attached to the discrete version in any way. An efficient algorithm for recovering linearly independent mixture components for discrete sample spaces with 3 samples per random group is described in \cite{anandkumar14}. Interestingly, with one more sample per group, these mixtures of measures become determined.
	\begin{thm} \label{thm:lidet}
		\sloppy If $\sP = \sum_{i=1}^m  w_i\delta_{\mu_i}$ is a mixture of measures where $\mu_1,\ldots,\mu_m$ are linearly independent then $\sP$ is $4$-determined.
	\end{thm}
	This bound is tight as a result of Theorem \ref{thm:nodet} with $m=2$.

	Our final result is related to the ``separability condition'' found in \cite{donoho03}. The separability condition in the discrete case requires that, for each mixture component $\mu_i$, there exists $B_i\in \sF$ such that $\mu_i\left( B_i \right)>0$ and $\mu_j\left( B_i \right) = 0$ for all $i\neq j$. There exists a generalization of the separability condition, known as {\em joint irreducibility}.
	\begin{defin}
		A collection of probability measures $\mu_1,\ldots,\mu_m$ are said to be {\em jointly irreducible} if $\sum_{i=1}^{m} w_i \mu_i$ being a probability measure implies $w_i\ge0$. 
	\end{defin}
	In other words, any probability measure in the span of $\mu_1,\ldots,\mu_m$ must be a convex combination of those measures. It was shown in \cite{blanchard14} that separability implies joint irreducibility, but not visa-versa. In that paper it was also shown that joint irreducibility implies linear independence, but the converse does not hold. 
	\begin{thm} \label{thm:ji}
		\sloppy If $\sP= \sum_{i=1}^m w_i \delta_{\mu_i} $ is a mixture of measures where $\mu_1,\ldots, \mu_m$ are jointly irreducible then $\sP$ is $2$-determined.
	\end{thm}

	A straightforward consequence of the corollary of Theorem 1 in \cite{donoho03} is that any mixture of measures on a finite discrete space with jointly irreducible components is $2$-identifiable. The result in \cite{donoho03} is concerned with the uniqueness of nonnegative matrix factorizations and Theorem \ref{thm:ji}, when applied to a finite discrete space, can be posed as a special case of the result in \cite{donoho03}. In the context of nonnegative matrix factorization the result in \cite{donoho03} is significantly more general than our result. In another sense our result is more general since it applies to spaces where joint irreducibility and the separability condition are not equivalent. Furthermore \cite{donoho03} only implies that the mixture of measures in Theorem \ref{thm:ji} are identifiable. The determinedness result is, as far as we know, totally new.

        \section{Tensor Products of Hilbert Spaces}\label{sec:tensorproducts}
	Our proofs will rely heavily on the geometry of tensor products of Hilbert spaces which we will introduce in this section.

	\subsection{Overview of Tensor Products}
	First we introduce tensor products of Hilbert spaces. To our knowledge there does not exist a rigorous construction of the tensor product Hilbert space which is both succinct and intuitive. Because of this we will simply state some basic facts about tensor products of Hilbert spaces and hopefully instill some intuition for the uninitiated by way of example. A thorough treatment of tensor products of Hilbert spaces can be found in \cite{kadison83}.

	Let $H$ and $H'$ be Hilbert spaces. From these two Hilbert spaces the ``simple tensors'' are elements of the form $h\otimes h'$ with $h\in H$ and $h' \in H'$. We can treat the simple tensors as being the basis for some inner product space $H_0$, with the inner product of simple tensors satisfying
	\begin{eqnarray*}
		\l<h_1 \otimes h_1', h_2 \otimes h_2'\r> = \l<h_1,h_2\r>\l<h_1',h_2'\r>.
	\end{eqnarray*}
	The tensor product of $H$ and $H'$ is the completion of $H_0$ and is denoted $H\otimes H'$. To avoid potential confusion we note that the notation just described is standard in operator theory literature. In some literature our definition of $H_0$ is denoted as $H\otimes H'$ and our definition of $H \otimes H'$ is denoted $H \widehat{\otimes} H'$.

	As an illustrative example we consider the tensor product $L^2\left( \rn \right) \otimes L^2\left( \rn \right)$. It can be shown that there exists an isomorphism between $L^2\left( \rn \right) \otimes L^2\left( \rn \right)$ and $L^2(\rn^2)$ that maps the simple tensors to separable functions \cite{kadison83}, $f \otimes f' \mapsto f(\cdot)f'(\cdot)$. We can demonstrate this isomorphism with a simple example. Let $f,g,f',g'\in L^2\left( \rn \right)$. Taking the $L^2(\rn^2)$ inner product of $f(\cdot)f'(\cdot)$ and $g(\cdot)g'(\cdot)$ gives us 

	\begin{eqnarray*}
		\int\int \l(f(x)f'(y)\r)\l(g(x)g'(y\r)) dx dy 
		&=& \int f(x)g(x) dx \int f'(y)g'(y) dy\\
	 &=& \l<f,g\r>  \l<f',g'\r>\\
	 &=& \l<f\otimes f', g \otimes g'\r>.
	\end{eqnarray*}

	Beyond tensor product we will need to define tensor power. To begin we will first show that tensor products are, in a certain sense, associative. Let $H_1, H_2, H_3$ be Hilbert spaces. Proposition 2.6.5 in \cite{kadison83} states that there is a unique unitary operator, $U: (H_1 \otimes H_2)\otimes H_3 \to H_1 \otimes (H_2 \otimes H_3)$, that satisfies the following for all $h_1 \in H_1, h_2 \in H_2, h_3 \in H_3$,
	\begin{eqnarray*}
		U\left( \left( h_1 \otimes h_2 \right)\otimes h_3 \right) = h_1 \otimes \left( h_2 \otimes h_3 \right).
	\end{eqnarray*}
        This implies that for any collection of Hilbert spaces, $H_1,\ldots , H_n$, the Hilbert space $H_1 \otimes \cdots \otimes H_n$ is defined unambiguously regardless of how we decide to associate the products. In the space $H_1 \otimes \cdots \otimes H_n$ we define a {\it simple tensor} as a vector of the form $h_1 \otimes\cdots\otimes h_n$ with $h_i \in H_i$. In \cite{kadison83} it is shown that $H_1 \otimes\cdots \otimes H_n$ is the closure of the span of these simple tensors. To conclude this primer on tensor products we introduce the following notation. For a Hilbert space $H$ we denote $H^{\otimes n}= \underbrace{H\otimes H \otimes \dots \otimes H}_\text{n times}$ and for $h \in H$, $h^{\otimes n}= \underbrace{h\otimes h \otimes \dots \otimes h}_\text{n times}$.
	\subsection{Tensor Rank}
	A tool we will use frequently in our proofs is {\em tensor rank}, which is similar to matrix rank.
	\begin{defin} \label{def:tensrank}
		Let $h\in H^{\otimes n}$ where $H$ is a Hilbert space. The {\em rank} of $h$ is the smallest natural number $r$ such that $h =\sum_{i=1}^r h_i$ where $h_i$ are simple tensors.
	\end{defin}
	In an infinite dimensional Hilbert space it is possible for a tensor to have infinite rank. We will only be concerned with finite rank tensors.

	\subsection{Some Results for Tensor Product Spaces}
	We derive some technical results concerning tensor product spaces that will be useful for the rest of the paper. These lemmas are similar to or  are straightforward extensions of previous results which we needed to modify for our particular purposes. Let $\left( \Psi, \sG, \gamma \right)$ be a $\sigma$-finite measure space. We have the following lemma that connects tensor power of a $L^2$ space to the $L^2$ space of the product measure. Proofs of many of the lemmas in this paper are deferred to the appendix. 
	\begin{lem}
		\label{lem:l2prod}
		There exists a unitary transform $U:L^2\left( \Psi, \sG, \gamma \right)^{\otimes n} \to L^2\left( \Psi^{\times n}, \sG^{\times n}, \gamma^{\times n} \right)$ such that, for all $f_1,\ldots, f_n \in L^2\left( \Psi, \sG, \gamma \right)$,
		\begin{align*}
			U\left( f_1\otimes \cdots \otimes f_n \right) = f_1(\cdot)\cdots f_n(\cdot).
		\end{align*}
	\end{lem}
	The following lemma is used in the proof of Lemma \ref{lem:l2prod} as well as the proof of Theorem \ref{thm:noident}.
	\begin{lem} \label{lem:unitprod}
		Let $H_1,\ldots, H_n, H_1',\ldots, H_n'$ be a collection of Hilbert spaces and $U_1,\ldots,U_n$ a collection of unitary operators with $U_i:H_i \to H_i'$ for all $i$. There exists a unitary operator $U:H_1 \otimes \cdots \otimes H_n \to H_1' \otimes \cdots \otimes H_n'$ satisfying $U\left( h_1 \otimes\cdots \otimes h_n \right) = U_1(h_1) \otimes \cdots \otimes U_n(h_n)$ for all $h_1 \in H_1 ,\ldots, h_n \in H_n$.
	\end{lem}
	A statement of the following lemma for $\rn^d$ can be found in \cite{symtensorrank}. We present our own proof for the Hilbert space setting in the appendix.
	\begin{lem}\label{lem:linind}
		Let $n>1$ and let $h_1,\ldots, h_n$ be elements of a Hilbert space such that no elements are zero and no pairs of elements are collinear. Then $h_1^{\otimes n-1},\ldots, h_n^{\otimes n-1}$ are linearly independent.
	\end{lem}
	The following lemma is a Hilbert space version of a well known property for positive semi-definite matrices.
	\begin{lem} \label{lem:tensrank}
		Let $h_1,\ldots,h_m$ be elements of a Hilbert space. The rank of $\sum_{i=1}^m h_i^{\otimes 2}$ is the dimension of $\spn\left( \l\{h_1,\ldots,h_m \r\}\right)$.
	\end{lem}
        \section{Proofs of Theorems}\label{sec:proofsoftheorems}
	With the tools developed in the previous sections we can now prove our theorems. First we introduce one additional piece of notation. For a function $f$ on a domain $\sX$ we define $f^{\times k}$ as simply the product of the function $k$ times on the domain $\sX^{\times k}$, $\underbrace{f(\cdot)\cdots f(\cdot)}_{\text{k times}}$. For a set, $\sigma$-algebra, or measure the notation continues to denote the standard $k$-fold product.

	In these proofs we will be making extensive use of various $L^2$ spaces. These spaces will be equivalence classes of functions which are equal almost everywhere with respect to the measure associated with that space. When considering elements of these spaces, equality will always mean almost everywhere equality with respect to the measure associated with that space. When performing integrals or other manipulations of elements in $L^2$ spaces, we will be performing operations that do not depend on the representative of the equivalence class.
	The following lemma will be quite useful.
	\begin{lem}\label{lem:noidentdown}
		Let $\gamma_1\ldots,\gamma_m, \pi_1\ldots,\pi_l$ be probability measures on a measurable space $\left( \Psi,\sG \right)$, $a_1\ldots,a_m,b_1,\ldots b_l \in \rn$, and $n\in \mathbb{N}_+$. If
		\begin{eqnarray*}
			\sum_{i=1}^m a_i \gamma_i^{\times n} =  \sum_{j=1}^l b_j \pi_j^{\times n}
		\end{eqnarray*}
		then for all $n'\in \mathbb{N}_+$ with $n'\le n$ we have that
		\begin{eqnarray*}
			\sum_{i=1}^m a_i \gamma_i^{\times n'} =  \sum_{j=1}^l b_j \pi_j^{\times n'}.
		\end{eqnarray*}
	\end{lem}
	\begin{proof}[Proof of Theorem \ref{thm:ident}]
		We proceed by contradiction. Suppose there exist $m,l\in \mathbb{N}_+$ with $l\le m$ such that there two different mixtures of measures $\sP = \sum_{i=1}^m a_i\delta_{\mu_i}  \neq \sP' = \sum_{j=1}^l b_j\delta_{\nu_j}$, and
		\begin{eqnarray*} 
			\sum_{i=1}^{m} a_i {\mu_i}^{\times 2m-1} = \sum_{j=1}^{l} b_j {\nu}_j^{\times 2m-1}. 
		\end{eqnarray*}
		Clearly $m>1$ otherwise we immediately arrive at a contradiction. By the well-ordering principle there exists a minimal $m$ such that the previous statement holds. For that minimal $m$ there exists a minimal $l$ such that the previous statement holds. We will assume that the $m$ and $l$ are both minimal in this way. This assumption implies that $\mu_i\neq \nu_j$ for all $i,j$. To prove this we will assume that there exists $i,j$ such that $\mu_i = \nu_j$, and show that this assumption leads to a contradiction. Without loss of generality we will assume that $\mu_m = \nu_l$. We will consider the three cases where $a_m=b_l$, $a_m>b_l$, and $a_m<b_l$.
		\begin{description}	
			\item[Case 1.] If $a_m = b_l$ then we have that
				\begin{eqnarray*}
					\sum_{i = 1}^{m-1}\frac{a_i}{1-a_m} \mu_i^{\times 2m-1} = \sum_{j=1}^{l-1}\frac{b_j}{1-b_l}\nu^{\times 2m-1}
				\end{eqnarray*}
				and from Lemma \ref{lem:noidentdown} we have
				\begin{eqnarray*}
					\sum_{i = 1}^{m-1}\frac{a_i}{1-a_m} \mu_i^{\times 2(m-1)-1} = \sum_{j=1}^{l-1}\frac{b_j}{1-b_l}\nu^{\times 2(m-1)-1}.
				\end{eqnarray*}
				Setting $\sP = \sum_{i=1}^{m-1} \frac{a_i}{1-a_m} \delta_{\mu_i}$ and $\sP' = \sum_{j=1}^{l-1} \frac{b_j}{1-b_l} \delta_{\nu_j}$, we have that $V_{2\left( m-1 \right) - 1} \left( \sP \right) = V_{2\left( m-1 \right)-1}\left( \sP' \right)$ which contradicts the minimality of $m$.
			\item[Case 2.]If $a_m >b_l$ then we have
				\begin{eqnarray*}
					\sum_{i=1}^{m-1} \frac{a_i}{1-b_l} \mu_i^{\times 2m-1} + \frac{a_m - b_l}{1-b_l}\mu_m^{\times 2m-1} = \sum_{j=1}^{l-1} \frac{b_j}{1-b_l}\nu_j^{\times 2m-1}
				\end{eqnarray*}
				which contradicts the minimality of $l$ by an argument similar to that in Case 1.
			\item[Case 3] If $a_m<b_l$ we have that
				\begin{eqnarray*}
					\sum_{i=1}^{m-1} \frac{a_i}{1-a_m} \mu_i^{\times 2m-1} = \sum_{j=1}^{l-1} \frac{b_j}{1-a_m} \nu_{j}^{\times 2m-1} + \frac{b_l - a_m}{1-a_m} \nu_l^{\times 2m-1}.
				\end{eqnarray*}
				Again we will use arguments similar to the one used in Case 1. If $l=m$ then swapping the mixtures associated with $m$ and $l$ gives us a pair of mixtures of measures which violates the minimality of $l$. If $l<m$ then from Lemma \ref{lem:noidentdown} we have that
				\begin{eqnarray*}
					\sum_{i=1}^{m-1} \frac{a_i}{1-a_m} \mu_i^{\times 2(m-1)-1} = \sum_{j=1}^{l-1} \frac{b_j}{1-a_m} \nu_{j}^{\times 2(m-1)-1} + \frac{b_l - a_m}{1-a_m} \nu_l^{\times 2(m-1)-1},
				\end{eqnarray*}
				which violates the minimality of $m$.
		\end{description}

		We have now established that $\mu_i \neq \nu_j$, for all $i,j$.
		We will use the following lemma to embed the mixture components in a Hilbert space.
		\begin{lem} \label{lem:himbed}
			Let $\gamma_1,\ldots,\gamma_n$ be finite measures on a measurable space $\left( \Psi, \sG \right)$. There exists a finite measure $\pi$ and non-negative functions $f_1,\ldots,f_n \in  L^1\left( \Psi,\sG,\pi \right)\cap L^2\left( \Psi,\sG,\pi \right)$ such that, for all $i$ and all $B \in \sG$
			\begin{eqnarray*}
				\gamma_i(B)=\int_B f_i d\pi. 
			\end{eqnarray*}
		\end{lem}
		From Lemma \ref{lem:himbed} there exists a finite measure $\xi$ and non-negative functions $p_1,\ldots,p_m,q_1,\ldots,q_l \in L^1\left( \Omega, \sF, \xi \right)\cap L^2\left( \Omega, \sF, \xi \right)$ such that, for all $B\in \sF$, $\mu_i(B) = \int_B p_i d\xi$ and $\nu_j(B) = \int_B q_j d\xi$ for all $i,j$. Clearly no two of these functions are equal (in the $\xi$-almost everywhere sense). If one of the functions were a scalar multiple of another, for example $p_1 = \alpha p_2$ for some $\alpha \neq 1$, it would imply
		\begin{eqnarray*}
			\mu_1\left( \Omega \right) = \int p_1 d\xi = \int \alpha p_2 d\xi=\alpha.
		\end{eqnarray*}
		This is not true so no pair of these functions are collinear.

		We can use the following lemma to extend this new representation to a product measure.

		\begin{lem} \label{lem:radprod}
			Let $\left( \Psi, \sG \right)$ be a measurable space, $\gamma$ and $\pi$ a pair of finite measures on that space, and $f$ a nonnegative function in $L^1\left(\Psi,\sG, \pi \right)$ such that, for all $A \in \sG$, $\gamma\left( A \right)=\int_A f d\pi$. Then for all $n$, for all $B \in \sG^{\times n}$ we have
			\begin{eqnarray*}
				\gamma^{\times n}\left( B \right) = \int_B f^{\times n} d\pi^{\times n}.
			\end{eqnarray*}
		\end{lem}
		Thus for any $R \in \sF^{\times 2m-1}$  we have
		\begin{eqnarray*}
			\int_R \sum_{i=1}^{m} a_i p_i^{\times 2m-1} d\xi^{\times 2m-1} 
			&=&  \sum_{i=1}^{m} a_i \mu_i^{\times 2m-1}\left( R \right)\\
		 &=&  \sum_{j=1}^{l} b_j \nu_j^{\times 2m-1}\left( R \right)\\
		 &=& \int_R \sum_{j=1}^{l} b_j q_j^{\times 2m-1}d\xi^{\times 2m-1}.
		\end{eqnarray*}
		The following lemma is a well known result in real analysis (Proposition 2.23 in \cite{folland99}), but it is worth mentioning explicitly.
		\begin{lem} \label{lem:inteq}
			Let $\left( \Psi,\sG,\gamma \right)$ be a measure space and $f,g \in L^1\left( \Psi,\sG,\gamma \right)$. Then $f=g$ $\gamma$-almost everywhere iff, for all $A\in \sG$, $\int_A f d\gamma = \int_A g d\gamma$.
		\end{lem}
		From this lemma it follows that
		\begin{eqnarray*}
			\sum_{i=1}^{m} a_i p_i^{\times 2m-1} = \sum_{j=1}^{l} b_j q_j^{\times 2m-1}.
		\end{eqnarray*}
		Applying the $U^{-1}$ operator from Lemma \ref{lem:l2prod} to the previous equation yields
		\begin{eqnarray*}
			\sum_{i=1}^{m} a_i p_i^{\otimes 2m-1} = \sum_{j=1}^{l} b_j q_j^{\otimes 2m-1}.
		\end{eqnarray*}
		Since $l+m \le2m$ Lemma \ref{lem:linind} states that 
		\begin{align*}
			p_1^{\otimes 2m-1},\ldots,p_{m}^{\otimes 2m-1},q_1^{\otimes 2m-1},\ldots,q_{l}^{\otimes 2m-1}
		\end{align*}
		are all linearly independent and thus $a_i = 0$ and $b_j = 0$ for all $i,j$, a contradiction.
	\end{proof}

	\begin{proof}[Proof of Theorem \ref{thm:noident}]
		To prove this theorem we will construct a pair of mixture of measures, $\sP \neq \sP'$ which both contain $m$ components and satisfy $V_{2m-2}\left( \sP \right) = V_{2m-2}\left( \sP' \right)$. From our definition of $\left( \Omega, \sF \right)$ we know there exists $F\in \sF$ such that $F$ and $F^C$ are nonempty. Let $x\in F$ and $x' \in F^C$. It follows that $\delta_{x}$ and $\delta_{x'}$ are different probability measures on $\left( \Omega, \sF \right)$. The theorem follows from the next lemma. We will prove the lemma after the theorem proof.
		\begin{lem}\label{lem:test}
			Let $\left( \Psi,\sG\right)$ be a measurable space and $\gamma, \gamma'$ be distinct probability measures on that space. Let $\varepsilon_1,\ldots,\varepsilon_t$ be $t\ge 3$ distinct values in $\left[ 0,1 \right]$. Then there exist $\beta_1,\ldots,\beta_t$, a permutation $\sigma:\left[ t \right]\to \left[ t \right]$, and $l\in \mathbb{N}_+$ such that
			\begin{eqnarray*}
				\sum_{i=1}^l \beta_i \left( \varepsilon_{\sigma\left( i \right)} \gamma + \left( 1-\varepsilon_{\sigma\left( i \right)} \right)\gamma' \right)^{\times t-2} = \sum_{j=l+1}^t \beta_j \left( \varepsilon_{\sigma\left( j \right)} \gamma + \left( 1-\varepsilon_{\sigma\left( j \right)} \right)\gamma' \right)^{\times t-2}
			\end{eqnarray*}
			where $\beta_i > 0$ for all $i$, $\sum_{i=1}^l \beta_i = \sum_{j=l+1}^t \beta_j = 1$, and $l,t- l \ge \left\lfloor\frac{t}{2}\right\rfloor$. 
		\end{lem}

		Let $\varepsilon_1,\ldots ,\varepsilon_{2m} \in \left[ 0,1 \right]$ be distinct and let $\mu_i = \varepsilon_i \delta_x + \left( 1-\varepsilon_i \right)\delta_{x'}$ for $i\in \left[ 2m \right]$. From Lemma \ref{lem:test} with $t = 2m$ there exists a permutation $\sigma: \left[ 2m \right] \to \left[ 2m \right]$ and $\beta_1,\ldots,\beta_{2m}$ such that
		\begin{eqnarray*}
			\sum_{i=1}^m \beta_i \mu_{\sigma\left( i \right)}^{\times 2m-2} = \sum_{j=m+1}^{2m} \beta_j \mu_{\sigma\left( j \right)}^{\times 2m-2},
		\end{eqnarray*}
		with $\sum_{i=1}^m \beta_i = \sum_{j=m+1}^{2m}\beta_j = 1$ and $\beta_i >0$ for all $i$.

		If we let $\sP = \sum_{i=1}^m \beta_i \delta_{\mu_{\sigma\left( i \right)}}$ and $\sP' = \sum_{j=m+1}^{2m} \beta_j \delta_{\mu_{\sigma\left( j \right)}}$, we have that $V_{2m-2}\left( \sP \right)= V_{2m-2}\left( \sP' \right)$ and $\sP \neq \sP'$ since $\mu_1,\ldots, \mu_{2m}$ are distinct.
	\end{proof}
	For the next proof we will introduce some notation. For a tensor $U \in \rn^{d_1}\otimes \cdots \otimes \rn^{d_l}$ we define $U_{i_1,\ldots,i_l}$ to be the entry in the $\left[ i_1,\ldots,i_l \right]$ location of $U$.

	\begin{proof}[Proof of Lemma \ref{lem:test}]
		From Lemma \ref{lem:himbed}, there exists a finite measure $\pi$ and non-negative functions $f,f' \in L^1\left( \Psi,\sG,\pi \right)\cap L^2\left( \Psi,\sG,\pi \right)$ such that, for all $A \in \sG$, $\gamma\left( A \right) = \int_A f d\pi$ and $\gamma'\left( A \right) = \int_A f' d\pi$.

		Let $H_2$ be the Hilbert space associated with the subspace in $L^2\left( \Psi,\sG,\pi \right)$ spanned by $f$ and $f'$. Let $\left( f_i \right)_{i=1}^{t}$ be non-negative functions in $L^1(\Psi, \sG, \pi)\cap L^2(\Psi, \sG, \pi)$ with $f_i = \varepsilon_i f + \left( 1-\varepsilon_i \right)f'$. Clearly $f_i$ is a pdf over $\pi$ for all $i$ and there are no pair in this collection which are collinear. Since $H_2$ is isomorphic to $\rn^2$ there exists a unitary operator $U:H_2 \to \rn^2$. From Lemma \ref{lem:unitprod} there exists a unitary operator $U_{t-2}:H_2^{\otimes t-2} \to {\rn^2}^{\otimes t-2}$, with $U_{t-2}\left( h_1 \otimes\cdots \otimes h_{t-2} \right) = U(h_1) \otimes \cdots \otimes U(h_{t-2})$. Because $U$ is unitary it follows that 
		\begin{align*}
			U_{t-2}\left( \spn\left( \left\{ h^{\otimes t-2}: h \in H_2 \right\} \right) \right) = \spn\left(\l\{ x^{\otimes t-2}:x \in \rn^2 \r\}\right).
		\end{align*}
		An order $r$ tensor, $A_{i_1,\ldots,i_r}$, is {\it symmetric} if $A_{i_1,\ldots,i_r} = A_{i_{\psi\left( 1 \right)},\ldots,i_{\psi\left( r \right)}}$for any $i_1,\ldots, i_r$ and permutation $\psi:\left[ r \right]\to \left[ r \right]$. A consequence of Lemma 4.2 in \cite{symtensorrank} is that $\spn\left( \l\{x^{\otimes t-2}:x \in \rn^2\r\} \right)\subset S^{ t-2}(\cn^2)$, the space of all symmetric order $t-2$ tensors over $\cn^2$. Complex symmetric tensor spaces will always be viewed as a vector space over the complex numbers and real symmetric tensor spaces will be always be viewed as a vector space over the real numbers.

	From Proposition 3.4 in \cite{symtensorrank} it follows that the dimension of $S^{ t-2 }\left( \cn^2 \right)$ is $\left( \begin{array}{c} 2+ t-2-1 \\ t-2 \end{array} \right) = t-1$. From this it follows that $\dim S^{t-2}\left( \rn^2 \right) \le t-1$, where $S^{t-2}\left( \rn^2 \right)$ is the space of all symmetric order $t-2$ tensors over $\rn^2$. To see this consider some set of linearly dependent tensors $x_1,\ldots,x_r  \in S^{t-2}\left( \mathbb{C}^2 \right)$ each containing only real valued entries, i.e. the tensors are in $S^{t-2}\l(\rn^2\r)$. Then it follows that there exists $c_1,\ldots, c_r \in \mathbb{C}$ such that
		\begin{eqnarray*}
			\sum_{i=1}^r c_i x_i = 0.
		\end{eqnarray*}
		Let $\Re$ denote the real component when applied to an element of $\mathbb{C}$, and the real component applied entrywise when applied to a tensor. We have that
		\begin{align*}
			0 =\Re\left(\sum_{i=1}^r  c_i  x_i\right) = \sum_{i=1}^r \Re \left(c_i  x_i\right) = \sum_{i=1}^r \Re \left(c_i  \right) x_i.
		\end{align*}
		Thus it follows that $x_1,\ldots x_r$ are linearly dependent in $S^{t-2}\l(\rn^2\r)$ and thus the dimensionality bound holds, $\dim S^{t-2}\left( \rn^2 \right) \le t-1$.

		From this we get that 
		\begin{align*}
			\dim\left( \spn\left( \left\{ h^{\otimes t-2}: h \in H_2 \right\} \right)\right)\le t-1.
		\end{align*}

		The bound on the dimension of $\spn\left( \left\{ h^{\otimes t-2}: h \in H_2 \right\} \right)$ implies that $\left( f_i^{\otimes t-2} \right)_{i=1}^{t}$ are linearly dependent. Conversely Lemma \ref{lem:linind} implies that removing a single vector from $\left( f_i^{\otimes t-2} \right)_{i=1}^{t}$ yields a set of vectors which are linearly independent. It follows that there exists $\left( \alpha_i \right)_{i=1}^{t}$ with $\alpha_i \neq 0$ for all $i$ and
		\begin{eqnarray} \label{eqn:cram}
			\sum_{i=1}^{t}\alpha_i f_i^{\otimes t-2} = 0.
		\end{eqnarray}
		There exists a permutation $\sigma:[t] \to [t]$ such that $\alpha_{\sigma\left( i \right)}< 0$ for all $i \in \left[ l \right]$ and $\alpha_{\sigma\left( j \right)} >0$ for all $j>l$ with $l\le \l\lfloor \frac{t}{2}\r\rfloor$ (ensuring that $l\le \l\lfloor \frac{t}{2}\r\rfloor$ may also require multiplying (\ref{eqn:cram}) by $-1$). This $\sigma$ appears in the lemma statement, but for the remainder of the proof we will simply assume without loss of generality that $\alpha_i<0$ for $i\in \left[ l \right]$ with $l\le \l\lfloor \frac{t}{2}\r\rfloor$.

		From this we have 
		\begin{eqnarray}\label{eqn:foo}
			\sum_{i=1}^{l}-\alpha_i f_i^{\otimes t-2}=\sum_{j=l+1}^{t}\alpha_j f_j^{\otimes t-2}.
		\end{eqnarray}
		From Lemma \ref{lem:l2prod} we have
		\begin{eqnarray*}
			\sum_{i=1}^{l}-\alpha_i f_i^{\times t-2}=\sum_{j=l+1}^{t}\alpha_j f_j^{\times t-2}
		\end{eqnarray*}
		and thus
		\begin{eqnarray*}
			\int \sum_{i=1}^{l}-\alpha_i f_i^{\times t-2} d\pi^{\times t-2 }&=&\int \sum_{j=l+1}^{t}\alpha_j f_j^{\times t-2}d\pi^{\times t-2 }\\
			\Rightarrow \sum_{i=1}^{l}-\alpha_i &=&\sum_{j=l+1}^{t}\alpha_j.
		\end{eqnarray*}
		Let $r=\sum_{i=1}^{l}-\alpha_i$. We know $r >0$ so dividing both sides of (\ref{eqn:foo}) by $r$ gives us
		\begin{eqnarray*}
			\sum_{i=1}^{l}-\frac{\alpha_i}{r} f_i^{\otimes t-2}=\sum_{j=l+1}^{t}\frac{\alpha_j}{r} f_j^{\otimes t-2}
		\end{eqnarray*}
		where the left and the right side are convex combinations. Let $\left( \beta_i \right)_{i=1}^{t}$ be positive numbers with $\beta_i = \frac{-\alpha_i}{r}$ for $i \in \left[ l \right]$ and $\beta_j = \frac{\alpha_j}{r}$ for $j\in \left[ t \right] \setminus \left[ l \right]$. This gives us
		\begin{eqnarray} \label{eqn:noidenttens}
			\sum_{i=1}^{l}\beta_i f_i^{\otimes t-2}=\sum_{j=l+1}^{t}\beta_j f_j^{\otimes t-2}.
		\end{eqnarray}
		We will now consider 3 cases for the value of $t$.

		\begin{description}
                    \item[Case 1.]If $t=3$ then $l=1$ and $l,t-l \ge \lfloor\frac{t}{2}\rfloor$ is satisfied.

                    \item[Case 2.] If $t$ is divisible by two then we can do the following,
                        \begin{eqnarray*}
                            \sum_{i=1}^{l} \beta_i f_i^{\otimes \frac{t}{2}-1}\otimes f_i^{\otimes \frac{t}{2}-1}&=&\sum_{j=l+1}^{t}\beta_j f_j^{\otimes \frac{t}{2}-1}\otimes f_j^{\otimes \frac{t}{2}-1}.
                        \end{eqnarray*}
                        \sloppy Consider the elements in the last equation as order two tensors in $ {L^2\left( \Psi, \sG, \pi \right)}^{\otimes \frac{t}{2}-1} \otimes  {L^2\left( \Psi, \sG, \pi \right)}^{ \otimes \frac{t}{2}-1} $. From Lemma \ref{lem:linind} and Lemma \ref{lem:tensrank} we have that the RHS of the previous equation has rank at least $\frac{t}{2}$ and since $l\le \frac{t}{2}$ it follows that $l=\frac{t}{2}$. Again we have that $l,t-l \ge\lfloor \frac{t}{2}\rfloor$.

                    \item[Case 3.] If $t$ is greater than 3 and not divisible by 2 then we can apply Lemma \ref{lem:l2prod} to get
                        \begin{eqnarray*}
                            \int_\Psi \sum_{i=1}^l \beta_i f_i^{\times t-3}f_i(x) d\pi(x) &=& \int_\Psi \sum_{j=l+1}^{t}\beta_j f_j^{\times t-3} f_j(y) d\pi(y)\\
                            \Rightarrow  \sum_{i=1}^l \beta_i f_i^{\times t-3} &=&  \sum_{j=l+1}^{t}\beta_j f_j^{\times t-3}.
                        \end{eqnarray*}
                        Applying Lemma \ref{lem:l2prod} again we get
                        \begin{eqnarray} 
                            \notag \sum_{i=1}^l \beta_i f_i^{\otimes t-3 } &=&  \sum_{j=l+1}^{t}\beta_j f_j^{\otimes t-3 } \\
                            \Rightarrow\sum_{i=1}^l \beta_i f_i^{\otimes \frac{t-1}{2} -1 } \otimes f_i^{\otimes \frac{t-1}{2} -1 } &=&  \sum_{j=l+1}^{t}\beta_j f_j^{\otimes \frac{t-1}{2} -1 } \otimes f_j^{\otimes \frac{t-1}{2} -1 }.\label{eqn:something}
                        \end{eqnarray}
		Recall that $\l\lfloor \frac{t}{2} \r\rfloor \ge l$ so we also have that
		\begin{eqnarray*}
			\l\lfloor \frac{t}{2}\r\rfloor - l &\ge& 0 \\
			\Rightarrow \frac{t}{2} - l &\ge& -\frac{1}{2}\\
			\Rightarrow t-l &\ge& \frac{t-1}{2}.
		\end{eqnarray*}
		From Lemma \ref{lem:linind} and Lemma \ref{lem:tensrank} we have that the RHS of (\ref{eqn:something}) has rank at least $\frac{t-1}{2}$ and thus $l\ge \frac{t-1}{2}$. From this we have that $t-l,l\ge \l\lfloor \frac{t}{2}\r\rfloor$ once again.
                \end{description}
 So $l,t-l \ge \lfloor \frac{t}{2}\rfloor$ for any $t\ge 3$. Applying Lemma \ref{lem:l2prod} to (\ref{eqn:noidenttens}) we have 
		\begin{eqnarray*}
			\sum_{i=1}^l \beta_i f_i^{\times t-2} = \sum_{j=l+1}^{t}\beta_j f_j^{\times t-2}.
		\end{eqnarray*}
		From Lemma \ref{lem:radprod} we have
		\begin{eqnarray*}
			\sum_{i=1}^l  \beta_i \left( \varepsilon_i \gamma + \left( 1-\varepsilon_i \right) \gamma' \right)^{\times t-2}   &=& \sum_{j=l+1}^{t} \beta_j \left( \varepsilon_j \gamma + \left( 1-\varepsilon_j \right) \gamma' \right)^{\times t-2}.
		\end{eqnarray*}
	\end{proof}

	\begin{proof}[Proof of Theorem \ref{thm:det}]
		Let $\sP = \sum_{i=1}^m a_i \delta_{\mu_i}$ and $\sP' = \sum_{j=1}^{l}  b_j \delta_{\nu_j}$ be mixtures of measures such that $\sP' \neq \sP$. We will proceed by contradiction. Suppose that $\sum_{i=1}^m  a_i \mu_i^{\times 2m} = \sum_{j=1}^l b_j \nu_j^{\times 2m} $. From Theorem \ref{thm:ident} we know that $\sP$ is $2m-1$-identifiable and therefore $2m$-identifiable by Lemma \ref{lem:ident}. It follows that $l>m$. From Lemma \ref{lem:himbed} there exists a finite measure $\xi$ and non-negative functions $p_1,\ldots,p_m,q_1,\ldots,q_l \in L^1\left( \Omega, \sF, \xi \right)\cap L^2\left( \Omega, \sF, \xi \right)$ such that, for all $B\in \sF$, $\mu_i(B) = \int_B p_i d\xi$ and $\nu_j(B) = \int_B q_j d\xi$ for all $i,j$. Using Lemmas \ref{lem:radprod} and \ref{lem:inteq} we have
		\begin{eqnarray*}
			\sum_{i=1}^m a_i p_i^{\times 2m} = \sum_{j=1}^l b_j q_j^{\times 2m}.
		\end{eqnarray*}
		By Lemma \ref{lem:l2prod} we have
		\begin{eqnarray*}
			\sum_{i=1}^m a_i p_i^{\otimes 2m} = \sum_{j=1}^l b_j q_j^{\otimes 2m},
		\end{eqnarray*}
		and therefore
		\begin{eqnarray*}
			\sum_{i=1}^m a_i p_i^{\otimes m}\otimes p_i^{\otimes m} = \sum_{j=1}^l b_j q_j^{\otimes m}\otimes q_j^{\otimes m}.
		\end{eqnarray*}
		\sloppy Consider the elements in the last inequality as tensors in $ {L^2\left( \Omega, \sF, \xi \right)}^{\otimes m} \otimes  {L^2\left( \Omega, \sF, \xi \right)}^{ \otimes m} $. Since no pair of vectors in $p_1,\ldots,p_m$ are collinear, from Lemma \ref{lem:linind} and Lemma \ref{lem:tensrank} we know that the LHS has rank $m$. On the other hand, no pair of vectors $q_1,\ldots,q_l$ are collinear either, so Lemma \ref{lem:linind} says that there is a subset of $\left\{q_1^{\otimes m},\ldots, q_l^{\otimes m}\right\}$ which contains at least $m+1$ linearly independent elements. By Lemma \ref{lem:tensrank} it follows that the RHS has rank at least $m+1$, a contradiction.
	\end{proof}
	\begin{proof}[Proof of Theorem \ref{thm:nodet}]
		To prove this theorem we will construct a pair of mixture of measures, $\sP \neq \sP'$ which contain $m$ and $m+1$ components respectively and satisfy $V_{2m-1}\left( \sP \right) = V_{2m-1}\left( \sP' \right)$. From our definition of $\left( \Omega, \sF \right)$ we know there exists $F\in \sF$ such that $F, F^C$ are nonempty. Let $x\in F$ and $x' \in F^C$. It follows that $\delta_{x}$ and $\delta_{x'}$ are different probability measures on $\left( \Omega, \sF \right)$.
		Let $\varepsilon_1,\ldots,\varepsilon_{2m+1}$ be distinct values in $\left[ 0,1 \right]$. Applying Lemma \ref{lem:test}  with $t=2m+1$ and letting $\mu_i = \varepsilon_i \delta_x + \left( 1-\varepsilon_i \right)\delta_{x'}$, there exists a permutation $\sigma:\left[ 2m+1 \right] \to \left[ 2m+1 \right]$ and $\beta_1,\ldots,\beta_{2m+1}$, with $\beta_i>0$ for all $i$ and $\sum_{i=1}^m \beta_i = \sum_{j=m+1}^{2m+1} \beta_j = 1$, such that
		\begin{eqnarray*}
			\sum_{i=1}^m \beta_i \mu_{\sigma\left( i \right)}^{\times 2m-1} = \sum_{j=m+1}^{2m+1} \beta_j \mu_{\sigma\left( j \right)}^{\times 2m-1}.
		\end{eqnarray*}
		If we let $\sP = \sum_{i=1}^m \beta_i \delta_{\mu_{\sigma\left( i \right)}}$ and $\sP' = \sum_{j=m+1}^{2m+1} \beta_j \delta_{\mu_{\sigma\left( j \right)}}$, we have that $V_{2m-1}\left( \sP \right)= V_{2m-1}\left( \sP' \right)$.
	\end{proof}
	To prove the remaining theorems we will need to make use of bounded linear operators on Hilbert spaces. Given a pair of Hilbert spaces $H,H'$ we define $\sL(H,H')$ as the space of {\em bounded linear operators} from $H$ to $H'$. An operator, $T$, is in this space if there exists a nonnegative number $C$ such that $\l\|Tx\r\|_{H'}\le C\l\|x\r\|_{H}$ for all $x \in H$. The space of bounded linear operators is a Banach space when equipped with the norm 
	\begin{eqnarray*}
		\l\|T\r\| \triangleq \sup_{x\neq 0}\frac{\l\|Tx\r\|}{\l\|x\r\|}.
	\end{eqnarray*}
	We will also need to employ Hilbert-Schmidt operators which are a subspace of the bounded linear operators.
	\begin{defin} \label{def:hs}
		Let $H,H'$ be Hilbert spaces and $T\in \sL\left( H,H' \right)$. $T$ is called a {\em Hilbert-Schmidt operator} if $\sum_{x\in J} \l\|T x \r\|^2 < \infty$ for an orthonormal basis $J\subset H$. We denote the set of Hilbert-Schmidt operators in $\sL\left( H,H' \right)$ by $\hs\left( H,H' \right)$.
	\end{defin}
	This definition does not depend on the choice of orthonormal basis: the sum $\sum_{x\in J} \l\|T\left( x \right)\r\|^2$ will always yield the same value regardless of the choice of orthonormal basis $J$. 

	The following properties of Hilbert-Schmidt operators will not be used in the next proof, but they will be useful later. The set of Hilbert-Schmidt operators is itself a Hilbert space when equipped with the inner product
	\begin{eqnarray*}
		\sum_{x\in J} \l<Tx,Sx\r>
	\end{eqnarray*}
	where $J$ is an orthonormal basis. Again this value does not depend on the choice of $J$. The Hilbert-Schmidt norm will be denoted as $\l\|\cdot \r\|_\hs$ and the standard operator norm will have no subscript. There is a well known bound relating the two norms: for a Hilbert-Schmidt operator $T$ we have that 
	\begin{eqnarray*}
		\l\|T \r\| \le \l\|T \r\|_\hs.
	\end{eqnarray*}

	\begin{proof}[Proof of Theorem \ref{thm:liident}]
		Let $\sP = \sum_{i=1}^m a_i \delta_{\mu_i}$ be a mixture of measures with linearly independent components. Let $\sP' = \sum_{j=1}^l  b_j \delta_{\nu_j}$ be a mixture of measures with $V_3(\sP) = V_3(\sP')$ and $l\le m$. From Lemma \ref{lem:himbed} there exists a finite measure $\xi$ and non-negative functions $p_1,\ldots,p_m,q_1,\ldots,q_l \in L^1\left( \Omega, \sF, \xi \right)\cap L^2\left( \Omega, \sF, \xi \right)$ such that, for all $B\in \sF$, $\int_B p_i d\xi = \mu_i(B)$ and $\int_B q_j d\xi = \nu_j\left( B \right)$ for all $i,j$. Using Lemma \ref{lem:noidentdown}, \ref{lem:radprod} , and \ref{lem:inteq} as we did in the previous theorem proofs it follows that
		\begin{eqnarray*}
			\sum_{i=1}^m a_i p_i^{\times 2}   &=& \sum_{j=1}^l b_j q_j^{\times 2}.
		\end{eqnarray*}
		From Lemma \ref{lem:l2prod} we have
		\begin{eqnarray*}
			\sum_{i=1}^m a_i p_i^{\otimes 2}   &=& \sum_{j=1}^l b_j q_j^{\otimes 2}.
		\end{eqnarray*}
		By Lemma \ref{lem:tensrank} we now know that the rank of the LHS of the previous equation is $m$ and thus thus $l=m$ and $q_1,\ldots,q_m$ are linearly independent.
		We will now show that $q_j\in \spn\left( \l\{p_1,\ldots,p_m \r\}\right)$ for all $j$. Suppose that $q_t \notin \spn\left(\l\{ p_1,\ldots,p_m \r\}\right)$. Then there exists $z\in L^2\left( \Omega,\sF,\xi \right)$ such that $z\perp p_1,\ldots ,p_m$ but $z \not \perp q_t$. Now we have 
		\begin{eqnarray*}
			\sum_{i=1}^m a_i p_i^{\otimes 2}   &=& \sum_{j=1}^m b_j q_j^{\otimes 2}\\
			\Rightarrow \l<\sum_{i=1}^m a_i p_i\otimes p_i, z\otimes z\r>   &=& \l<\sum_{j=1}^m b_j q_j \otimes q_j,z\otimes z\r> \\
			\Rightarrow \sum_{i=1}^m a_i \l< p_i \otimes p_i, z\otimes z\r>   &=& \sum_{j=1}^m b_j \l<  q_j\otimes q_j,z\otimes z\r> \\
			\Rightarrow \sum_{i=1}^m a_i \l< p_i , z\r>^2   &=& \sum_{j=1}^m b_j \l<  q_j,z\r>^2.
		\end{eqnarray*}
		We know that the LHS of the last equation is zero but the RHS is not, a contradiction.

		We will find the following well known property of tensor products to be useful for continuing the proof (\cite{kadison83} Proposition 2.6.9).
		\begin{lem} \label{lem:hstens}
			Let $H,H'$ be Hilbert spaces. There exists a unitary operator $U:H\otimes H' \to \hs\left( H,H' \right)$ such that, for any simple tensor $h\otimes h' \in H\otimes H'$, $U\left( h\otimes h' \right) = \l<h,\cdot\r> h'$.
		\end{lem}
		Because $p_1,\ldots,p_m$ are linearly independent we can do the following: for each $k \in \left[ m \right]$ let $z_k \in \spn\left(\l\{ p_1,\ldots,p_m \r\}\right)$ be such that $z_k \perp \left\{ p_i: i \neq k \right\}$ and $\l<z_k,p_k\r> =1$. By considering elements of $L^2\left( \Omega,\sF, \xi \right)^{\otimes 3}$ as elements of $L^2\left( \Omega,\sF, \xi \right) \otimes L^2\left( \Omega,\sF, \xi \right)^{\otimes 2}$, we can use Lemma \ref{lem:hstens} to transform elements in $L^2\left( \Omega,\sF, \xi \right)^{\otimes 3}$ into elements of $\hs\left( L^2\left( \Omega,\sF, \xi \right),L^2\left( \Omega,\sF, \xi \right)^{\otimes 2} \right)$,
		\begin{eqnarray*}
			\sum_{i=1}^m a_i p_i^{\otimes 3}   &=& \sum_{j=1}^m b_j q_j^{\otimes 3} \\
			\Rightarrow \sum_{i=1}^m a_i p_i^{\otimes 2}\l<p_i,\cdot\r>   &=& \sum_{j=1}^m b_j q_j^{\otimes 2} \l<q_j,\cdot\r>.
		\end{eqnarray*}
		It now follows that
		\begin{eqnarray*}
			\notag \sum_{i=1}^m a_i p_i^{\otimes 2}\l<p_i,z_k\r>   &=& \sum_{j=1}^m b_j q_j^{\otimes 2} \l<q_j,z_k\r>\\
			\Rightarrow a_k p_k^{\otimes 2}  &=& \sum_{j=1}^m b_j q_j^{\otimes 2} \l<q_j,z_k\r>.
		\end{eqnarray*}
		Using Lemma \ref{lem:hstens} we have
		\begin{eqnarray}\label{eqn:ranksum}
			a_k p_k\l<p_k,\cdot\r>  &=& \sum_{j=1}^m b_j\l<q_j,z_k\r> q_j\l<q_j,\cdot\r> .
		\end{eqnarray}
		The LHS of (\ref{eqn:ranksum}) is a rank one operator and thus the RHS must have exactly one nonzero summand, since $q_1,\ldots,q_m$ are linearly independent. Let $\varphi:\left[ m \right]\to \left[ m \right]$ be a function such that, for all $k$,
		\begin{eqnarray*}
			a_k p_k^{\otimes 2} =\l<q_{\varphi\left( k \right)},z_k\r> b_{\varphi\left( k \right)} q_{\varphi\left( k \right)}^{\otimes 2}.
		\end{eqnarray*}
		From Lemma \ref{lem:radprod} we have
		\begin{eqnarray*}
			a_k \mu_k^{\times 2} &=& \l<q_{\varphi\left( k \right)},z_k\r>b_{\varphi\left( k \right)} \nu_{\varphi\left( k \right)}^{\times 2},
		\end{eqnarray*}
		for all $k$.
                By Lemma \ref{lem:noidentdown} we have that $a_k \mu_k = \l<q_{\varphi\left( k \right)},z_k\r>b_{\varphi\left( k \right)} \nu_{\varphi\left( k \right)}$ for all $k$ and thus $\mu_k = \nu_{\varphi\left( k \right)}$ since $\mu_k$ and $\nu_{\varphi\left( k \right)}$ are collinear probability measures. Because $\mu_i \neq \mu_j$ for all $i,j$ we have that $\varphi$ must be a bijection. Let $\sigma = \varphi^{-1}.$
		By Lemma \ref{lem:noidentdown} we have that
		\begin{eqnarray*}
			\sum_{i=1}^m a_i \mu_i = \sum_{j=1}^m b_j \mu_{\sigma\left( j \right)}.
		\end{eqnarray*}
		Since $\mu_1,\ldots,\mu_m$ are linearly independent the last equation only has one solution for $b_1,\ldots,b_m$, which is $b_k = a_{\sigma\left( k \right)}$, for all $k$. Thus
		\begin{eqnarray*}
			\sP' = \sum_{i=1}^m a_{\sigma\left( i \right)} \delta_{ \mu_{\sigma\left( i \right)}}
		\end{eqnarray*}
		which is equal to $\sP$.
	\end{proof}
	\begin{proof}[Proof of Theorem \ref{thm:lidet}]
		Let $\sP = \sum_{i=1}^m a_i \delta_{\mu_i}$ be a mixture of measures with linearly independent components. We will proceed by contradiction:  let $\sP' = \sum_{j=1}^l  b_j \delta_{\nu_j} \neq \sP$ be a mixture of measures with $V_4(\sP) = V_4(\sP')$. From Theorem \ref{thm:ident} we know that $\sP$ is $3$-identifiable. By Lemma \ref{lem:ident} it follows that $\sP$ is $4$-identifiable and thus $l>m$. From Lemma \ref{lem:himbed} there exists a finite measure $\xi$ and non-negative functions $p_1,\ldots,p_m,q_1,\ldots,q_l \in L^1\left( \Omega, \sF, \xi \right)\cap L^2\left( \Omega, \sF, \xi \right)$ such that, for all $B\in \sF$, $\int_B p_i d\xi = \mu_i(B)$ and $\int_B q_j d\xi = \nu_j\left( B \right)$ for all $i,j$.

		Proceeding as we did in the proof of Theorem \ref{thm:liident} we have that
		\begin{eqnarray*}
			\sum_{i=1}^m a_i p_i^{\otimes 4} = \sum_{j=1}^{l} b_j q_j^{\otimes 4}.
		\end{eqnarray*}
		Suppose that there exists $k$ such that $\nu_k \notin \spn\left( \left\{ \mu_1,\ldots,\mu_m \right\} \right)$. From this it would follow that there exists $z$ such that $z\perp \left\{ p_1,\ldots,p_m \right\}$ and $z \not \perp q_k$. Then we would have that 
		\begin{eqnarray*}
			\l<\sum_{i=1}^m a_i p_i^{\otimes 4},z^{\otimes 4}\r> &=& \l< \sum_{j=1}^{l} b_j q_j^{\otimes 4}, z^{\otimes 4}\r>\\
			\Rightarrow \sum_{i=1}^m a_i \l<p_i,z\r>^4 &=&  \sum_{j=1}^l b_j \l<q_j,z\r>^4,
		\end{eqnarray*}
		but the LHS of the last equation is 0 and the RHS is positive, a contradiction. Thus we have that $q_k \in \spn\left( \left\{ p_1,\ldots,p_m \right\} \right)$ for all $k$.

		Since $l>m$ and no pair of elements in $q_1,\ldots,q_m$ are collinear, there must a vector in $q_1,\ldots, q_l$ which is a nontrivial linear combination of $p_1,\ldots,p_m$. Without loss of generality we will assume that $q_1 = \sum_{i=1}^m c_i p_i$ with $c_1$ and $c_2$ nonzero. By the linear independence of $p_1,\ldots,p_m$ there must exist vectors $z_1,z_2$ such that $\l<z_1, p_1\r> =1$, $z_1 \perp \left\{ p_i : i\neq 1\right\}$, $\l<z_2,p_2\r> = 1$, and $z_2 \perp \left\{ p_i: i\neq 2 \right\}$. Now consider
		\begin{eqnarray*}
			\l<\sum_{i=1}^m a_i p_i^{\otimes 4},z_1^{\otimes 2} \otimes z_2^{\otimes 2}\r> &=& \l< \sum_{j=1}^{l} b_j q_j^{\otimes 4}, z_1^{\otimes 2} \otimes z_2^{\otimes 2}\r>\\
			\Rightarrow\sum_{i=1}^m a_i \l< p_i,z_1\r>^2 \l<p_i, z_2\r>^2 &=&\sum_{j=1}^{l} b_j\l<   q_j, z_1\r>^2 \l< q_j, z_2\r>^2.
		\end{eqnarray*}
		The LHS of the last equation is 0 and the RHS is positive, a contradiction.
	\end{proof}
	\begin{proof}[Proof of Theorem \ref{thm:ji}]
		Let $\sP = \sum_{i=1}^m a_i \delta_{\mu_i}$ be a mixture of measures with jointly irreducible components. Consider a mixture of measures $\sP' = \sum_{j=1}^l b_j\delta_{\nu_j}$ with $V_2(\sP)= V_2(\sP')$. From Lemma \ref{lem:himbed} there exists a finite measure $\xi$ and non-negative functions $p_1,\ldots,p_m,q_1,\ldots,q_l \in L^1\left( \Omega, \sF, \xi \right)\cap L^2\left( \Omega, \sF, \xi \right)$ such that, for all $B\in \sF$, $\int_B p_i d\xi = \mu_i(B)$ and $\int_B q_j d\xi = \nu_j\left( B \right)$ for all $i,j$. From Lemmas \ref{lem:radprod} and \ref{lem:inteq} we have
		\begin{eqnarray*}
			\sum_{i=1}^m  a_i p_i \times p_i = \sum_{j=1}^l  b_j q_j\times q_j.
		\end{eqnarray*}
		From Lemma \ref{lem:l2prod} we have
		\begin{eqnarray}
			\label{eqn:noident}
			\sum_{i=1}^m  a_i p_i \otimes p_i = \sum_{j=1}^l  b_j q_j\otimes q_j.
		\end{eqnarray}

		Suppose for a moment that $\sP'$ contains a mixture component which does not lie in $\spn\left(\l\{ \mu_1,\ldots,\mu_m \r\}\right)$. Without loss of generality we will assume that $\nu_1 \notin \spn\left(\l\{ \mu_1,\ldots,\mu_m \r\}\right)$. Recall that joint irreducibility implies linear independence so $\nu_1,\mu_1,\ldots,\mu_m$ are a linearly independent set of measures and thus $q_1,p_1,\ldots,p_m$ are linearly independent. It follows that we can find some $z\in L^2\left( \Omega,\sF,\xi \right)$ such that $\l<z,q_1\r> \neq 0$ and $z\perp \left\{ p_i: i\in \left[ m \right] \right\}$. From (\ref{eqn:noident}) we have the following
		\begin{eqnarray*}
			\l<\sum_{i=1}^m  a_i p_i \otimes p_i,z\otimes z\r> &=& \l<\sum_{j=1}^l  b_j q_j\otimes q_j,z\otimes z\r>\\
			\Rightarrow \sum_{i=1}^m a_i\l<   p_i \otimes p_i,z\otimes z\r> &=& \sum_{j=1}^l b_j\l<   q_j\otimes q_j,z\otimes z\r>\\
			\Rightarrow \sum_{i=1}^m a_i\l<   p_i ,z\r>^2 &=& \sum_{j=1}^l b_j\l<   q_j,z\r>^2.
		\end{eqnarray*}
		All the summands on both sides of the last equation are nonnegative. By our construction of $z$ the LHS of the previous equation is zero and the first summand on the RHS is positive, a contradiction. Thus, each component in $\sP'$ must lie in the span of the components of $\sP$.

		Now we have, for all $j$, $q_j = \sum_{i=1}^m c_i^j p_i$. From joint irreducibility we have that $c_i^j \ge 0$ for all $i$ and $j$. Now suppose that there exists $r,s,s'$ such that $c_s^r,c_{s'}^r >0$. From the linear independence of $p_1,\ldots,p_m$ we can find a $z$ such that $\l<p_s,z\r> = 1$ and $z\perp \left\{ p_q:q\in \left[ m \right]\setminus \left\{ s \right\} \right\}$. Applying Lemma \ref{lem:hstens} to (\ref{eqn:noident}) we have
		\begin{eqnarray*}
			\sum_{i=1}^m  a_i p_i \l< p_i, \cdot \r> &=& \sum_{j=1}^l  b_j q_j \l<q_j,\cdot\r>\\
			\Rightarrow \sum_{i=1}^m  a_i p_i \l< p_i, z \r> &=& \sum_{j=1}^l  b_j q_j \l<q_j,z\r>\\
			\Rightarrow  a_s p_s  &=& \sum_{j=1}^l  b_j \l[\sum_{t=1}^m c_t^j p_t\r] \l<\sum_{u=1}^m c_u^j p_u,z\r>\\
			\Rightarrow  a_s p_s  &=& \sum_{j=1}^l  b_j \l[\sum_{t=1}^m c_t^j p_t\r] c_s^j\\
                        &=&    \sum_{t=1}^m \sum_{j=1}^l b_jc_t^j c_s^j p_t \\
                        &=&    \sum_{t=1}^m p_t \sum_{j=1}^l b_jc_t^j c_s^j. 
		\end{eqnarray*}
		Let $\alpha_t = \sum_{j=1}^l b_jc_t^j c_s^j$ for all $t$ and note that each summand is nonnegative. Now we have
		\begin{eqnarray*}
			a_s p_s = \sum_{t=1}^m \alpha_t p_t.
		\end{eqnarray*}
		 We know that $\alpha_{s'}>0$ since $b_r c_s^r c_{s'}^r >0$. This violates the linear independence of $p_1,\ldots,p_m$. Now we have that for all $i$ there exists $j$ such that $p_i = q_j$. From the minimality of the representation of mixtures of measures it follows that $l=m$ and without loss of generality we can assert that $p_i = q_i$ for all $i$ and thus $\mu_i = \nu_i$ for all $i$. Because $p_1, \ldots, p_m$ are linearly independent it follows that $p_1\otimes p_1,\ldots, p_m\otimes p_m$ are linearly independent. We can show this by the contrapositive, suppose $p_1\otimes p_1, \ldots , p_m\otimes p_m$ are not linearly independent then there exists a nontrivial linear combination such that  $\sum_{i=1}^m \kappa_i p_i\otimes p_i = 0$. Assume without loss of generality that $\kappa_1 \neq 0$. Applying Lemma \ref{lem:hstens} we get that
		\begin{eqnarray*}
			\sum_{i=1}^m \kappa_i p_i \l<p_i,\cdot\r> &=&  0\\
			\Rightarrow \sum_{i=1}^m \kappa_i p_i \l<p_i,p_1\r> &=&  0\\
			\Rightarrow \kappa_1 p_1 \l\|p_1\r\|_{L^2}^2 + \sum_{i=2}^m \kappa_i p_i \l<p_i,p_1\r> &=&  0\\
		\end{eqnarray*}
		and thus $p_1,\ldots,p_m$ are not linearly independent.

		Since $p_1\otimes p_1,\ldots,p_m\otimes p_m$ are linearly independent it follows that $a_i = b_i$ for all $i$ and thus $\sP = \sP'$.

	\end{proof}

	\section{Identifiability and Determinedness of Mixtures of Multinomial Distributions}\label{sec:multinomial}
	Using the previous results we can show analogous identifiability and determinedness results for mixtures of multinomial distributions. The identifiability of mixtures of multinomial distributions was originally studied in \cite{kim1984} which contains a proof of Corollary \ref{cor:multident} from this paper. An alternative proof of this corollary can be found in \cite{elmore2003}. These results are analogous to identifiability results presented in this paper. Our proofs use techniques which are very different from those used in \cite{kim1984,elmore2003}. Our techniques can also be used to prove a determinedness style result, Corollary \ref{cor:multdet}, which we have not seen addressed elsewhere in the multinomial mixture model literature.

	Central to the results in this section is Lemma \ref{lem:mult} which establishes an equivalence between the grouped sample setting and multinomial mixture models. A sample from a multinomial distribution can be viewed as a totalling the outcomes from an iid sampling of a categorical distribution. Consider some probability measure $\mu$ over a finite discrete space and let $\bX=\left( X_1,\ldots,X_m \right)$ be a collection of $m$ iid samples from $\mu$. Here $\bX$ has the form of what we would call a ``random group.'' Because $\bX$ contains {\em iid} samples no useful statistical information is contained in the order of the samples. It follows that we can simply tally the number of results for each outcome and not lose any useful statistical information. Lemma \ref{lem:mult} formalizes this intuition so that we can apply tools developed earlier in this paper to the multinomial mixture model setting.

	Before our proofs we must first introduce some definitions and notation. Any multinomial distribution is completely characterized by positive integers $n$ and $q$ and a probability vector in $\rn^q$, $ p = \left[ p_1,\ldots,p_q \right]^T$.The value $q$ represents the number of possible outcomes of a trial, $p$ is the likelihood of each outcome on a trial, and $n$ is the number of trials. For whole numbers $k,l$ we define $C_{k,l}= \l\{x \in \nn^{\times l} : \sum_{i=1}^l x_i = k\r\}$. These are vectors of the form $\left[ x_1,\ldots,x_l \right]^T$ where $\sum_{i=1}^l x_i =k$. Using the values $n$ and $q$ above, the multinomial distribution is a probability measure over $C_{n,q}$. If $Q$ is a multinomial distribution with parameters $n,p,q$ as defined above then its probability mass function is 
	\begin{eqnarray*}
		Q\left(\l\{ \left[ x_1,\ldots,x_q \right]^T\r\} \right) =
		\frac{n!}{x_1! \cdots x_q !} p_1^{x_1}\cdots p_q^{x_q}
	\end{eqnarray*} 
	for $x \in C_{n,q}$.
	We will denote this measure as $Q_{n,p,q}$.
	Let 
	\begin{equation*}
		\sM\left( n,q \right)\triangleq \l\{Q_{n,p,q} : p \text{ is a probability vector in } \rn^q \r\},
	\end{equation*}
	i.e. the space of all multinomial distributions with $n$ and $q$ fixed.

To show identifiability and determinedness of mixtures of multinomial distributions we will construct a linear operator $T_{n,q}$ from $\spn \l(\sD\l(C_{n,q},2^{C_{n,q}}\right)\r)$ to $\spn\l(\sD\l(\left[ q \right]^{\times n}, 2^{\left[ q \right]^{\times n}}\r)\r)$ and use it to show that non-identifiable mixtures of multinomial distributions yield non-identifiable mixtures of measures. We will also use $T_{n,q}$ to show that non-determined mixtures of multinomial distributions yield non-determined mixtures of measures.

Since $C_{n,q}$ is a finite set, the vector space of finite signed measures on $\left( C_{n,q}, 2^{C_{n,q}} \right)$ is a finite dimensional space and the set $\left\{ \delta_x: x\in C_{n,q} \right\}$ is a basis for this space. Note that $\left\{ \delta_x: x\in C_{n,q} \right\}$ is the set of all {\em point masses} on $C_{n,q}$, not vectors in the ambient space of $C_{n,q}$. Thus, to completely define the operator $T_{n,q}$, we need only define $T_{n,q}\left( \delta_x \right)$ for all $x\in C_{n,q}$. To this end let $x \in C_{n,q}$. We define the function $F_{n,q}:C_{n,q} \to \left[ q \right]^{\times n}$ as $F_{n,q}\l(x\r) =1^{\times x_1}\times \cdots \times q^{\times x_q}$, where the exponents represent Cartesian powers. The definition of $F_{n,q}$ is a bit dense so we will do a simple example. Suppose $n=6,q=4$ and $x= \left[ 1,0,3,2 \right]^T$ then $F_{n,q}\l(x\r) = \left[ 1,3,3,3,4,4 \right]^T$. Intuitively the $F_{n,q}$ operator undoes the totalling which transforms a collection of trials from a categorical distribution into a draw from a multinomial distribuiton; $F_{ n,q }$ returns these trials in nondecreasing order. Let $S_n$ be the symmetric group on $n$ symbols. We define our linear operator as follows
\begin{eqnarray*}
	T_{n,q}\left( \delta_x \right) = \frac{1}{n!} \sum_{\sigma \in S_n} \delta_{ \sigma\l(F_{n,q}(x)\r)},
\end{eqnarray*}
where $\sigma$ is permuting the entries of $F_{n,q}\left( x \right)$.
This operator is similar to the projection operator onto the set of order $n$ symmetric tensors \cite{symtensorrank}. The following lemma makes the crucial connection between the space of multinomial distributions and the probability measures of grouped samples.
\begin{lem}\label{lem:mult}
	Let $Q_{n,p,q} \in \sM\left( n,q \right)$, then 
	\begin{eqnarray*}
		T_{n,q}\left( Q_{n,p,q} \right) = V_n\left( \delta_{\sum_{i=1}^q p_i \delta_i} \right).
	\end{eqnarray*}
\end{lem}
\begin{proof}[Proof of Lemma \ref{lem:mult}]
	For brevity's sake let $Q =T_{n,q}\left( Q_{n,p,q} \right)$ and  $R= V_n\left( \delta_{\sum_{i=1}^q p_i \delta_i} \right)$. Let $y \in \left[ q \right]^{\times n}$ be arbitrary. We will prove that $Q(\l\{y\r\}) = R(\l\{y\r\})$ which, since $y$ is arbitrary, clearly generalizes to $Q=R$. Let $\check{y}\in C_{n,q}$ be the element such that $\check{y}_i = \l|\left\{ j:y_j = i \right\}\r|$ for all $i$, i.e. the $i$th index of $\check{y}$ contains the number of times the value $i$ occurs in $y$. From the definition of $V_n$ we have that 
	\begin{align*}
		R(\l\{y\r\}) &= \left( \sum_{i=1}^q p_i \delta_i \right)^{\times n} \left( \left\{ y \right\} \right) = \prod_{i=1}^n p_{y_i} = \prod_{j=1}^{q}p_j^{\check{y}_j}.
	\end{align*}

	 We define $\chi$ to be the indicator function, which is equal to $1$ if its subscript is true and $0$ otherwise. Consider some $z\neq \check{y}$. We have
	\begin{eqnarray*}
		T_{n,q}\left( \delta_z \right)(\{y\}) 
		&=& \frac{1}{n!} \sum_{\sigma \in S_n} \delta_{ \sigma\l(F_{n,q}(z)\r)}\left( \left\{ y \right\} \right)\\
		&=& \frac{1}{n!} \sum_{\sigma \in S_n}  \chi_{\sigma\l(F_{n,q}(z)\r)= y}.
	\end{eqnarray*}
	From our definition of $F_{n,q}$ and $\check{y}$ it is clear that, there must exist some $r$ such that the number of entries of $F_{n,q}(z)$ which equal $r$ is different from the number of indices of $y$ which equal $r$. Because of this no permutation of $F_{n,q}(z)$ can equal $y$ and thus $T_{n,q}\left( \delta_z \right)(\{y\}) = 0$. From this it follows that $T_{n,q}\left( \delta_z \right)(\{y\}) = 0$ for all $z\neq \check{y}$.

	Now we will consider $T_{n,q}\left( \delta_{\check{y}} \right)(\{y\})$. Again we have
	\begin{eqnarray*}
		T_{n,q}\left( \delta_{\check{y}} \right)(\{y\}) =  \frac{1}{n!} \sum_{\sigma \in S_n}  \chi_{\sigma\l(F_{n,q}(\check{y})\r)= y},
	\end{eqnarray*}
	so we need only determine how many permutations of $F_{n,q}\left( \check{y} \right)$ are equal to $y$. Basic combinatorics tells us that there are $\check{y}_1!\cdots \check{y}_q!$ such permutations. The coefficient of $\delta_{\check{y}}$ in $Q_{n,p,q}$ is $\frac{n!}{\check{y}_1! \cdots \check{y}_q!} p_1^{\check{y}_1}\cdots p_n^{\check{y}_n}$ so we have that $Q(\l\{y\r\}) = R(\l\{y\r\})$ by direct evaluation.
\end{proof}

This lemma allows us to make some assertions about the identifiability of mixtures of multinomial distributions.

In the following we will assume that all multinomial mixture models under consideration have only nonzero summands and distinct components. In the context of multinomial mixture models, a multinomial mixture model $\sum_{i=1}^m a_i Q_{n,p_i,q}$ is identifiable if it being equal to a different multinomial mixture model,
\begin{eqnarray*}
	\sum_{i=1}^m a_i Q_{n,p_i,q} = \sum_{j=1}^s b_j Q_{n,r_j,q},
\end{eqnarray*}
with $s\le m$ implies that $s=m$ and there exists some permutation $\sigma$ such that $a_i = b_{\sigma\left( i \right)}$ and $Q_{n,p_i,q} = Q_{n,r_{\sigma\left( i \right)},q}$ for all $i$. The mixture model is determined if the previous statement holds without the restriction $s\le m$.

Multinomial mixture models are identifiable if the number of components $m$ and the number of trials in each component $n$ satisfy $n\ge 2m-1$.
\begin{cor}\label{cor:multident}
	Let $m\in \nn_+$, $n\ge2m-1$, and fix $q\in \nn_+$. Let $Q_{n,p_1,q},\ldots,Q_{n,p_m,q}, Q_{n,r_1,q},\ldots,Q_{n,r_s,q} \in \sM\left( n,q \right)$ with $Q_{n,p_1,q},\ldots,Q_{n,p_m,q}$ distinct, $Q_{n,r_1,q},\ldots,Q_{n,r_s,q}$ distinct, and $s\le m$. If
	\begin{eqnarray*}
		\sum_{i=1}^m a_i Q_{n,p_i,q} = \sum_{j=1}^s b_j Q_{n,r_j,q}
	\end{eqnarray*} 
	with $a_i>0,b_i> 0$ for all $i$ and $\sum_{i=1}^m a_i = \sum_{j=1}^s b_j = 1$, then $s = m$ and there exists some permutation $\sigma$ such that $a_i = b_{\sigma(i)}$ and $p_i = r_{\sigma(i)}$.
\end{cor}
Alternatively this corollary says that, given two different finite mixtures with components in $\sM(n,q)$, one mixture with $m$ components and the other with $s$ components, if $n\ge 2m-1$ and $n\ge 2s-1$ then the mixtures induce different measures.
\begin{proof}[Proof of Corollary \ref{cor:multident}]
	We will proceed by contradiction and assume that there exists two mixtures of the form above,
	\begin{eqnarray*}
		\sum_{i=1}^m a_i Q_{n,p_i,q} = \sum_{j=1}^s b_j Q_{n,r_j,q}
	\end{eqnarray*} 
	but $s\neq m$ or $s=m$ and there exists no permutation such that $a_i Q_{n,p_i,q} = b_{\sigma\left( i \right)} Q_{ n,r_{\sigma\left( i \right)},q}$. 
	If we apply $T_{n,q}$ defined earlier, from Lemma \ref{lem:mult} it follows that
	\begin{eqnarray*}
		V_n\left(\sum_{i=1}^{m} a_i \delta_{\sum_{k=1}^q p_{i,k} \delta_k} \right) = V_n\left(\sum_{j=1}^{s} b_j \delta_{\sum_{l=1}^q r_{j,l} \delta_l} \right).
	\end{eqnarray*}
	We have that $\sP = \sum_{i=1}^{m} a_i \delta_{\sum_{k=1}^q p_{i,k} \delta_k}$ and $\sP' = \sum_{j=1}^{s} b_j \delta_{\sum_{l=1}^q r_{j,l} \delta_l}$ are mixtures of measures which are not $n$-identifiable. Our contradiction hypothesis implies that $\sP \neq \sP'$. From Lemma \ref{lem:noident} we have that 
	\begin{eqnarray*}
		V_{2m-1}\left(\sum_{i=1}^{m} a_i \delta_{\sum_{k=1}^q p_{i,k} \delta_k} \right) = V_{2m-1}\left(\sum_{j=1}^{s} b_j \delta_{\sum_{l=1}^q r_{j,l} \delta_l} \right),
	\end{eqnarray*}
	which contradicts Theorem \ref{thm:ident}.
\end{proof}
Additionally multinomial mixture models are determined if the number of components $m$ and the number of trials in each component $n$ satisfy $n\ge 2m$.

\begin{cor} \label{cor:multdet}
	Let $n\ge2m$ and fix $q\in \nn$. Let $Q_{n,p_1,q},\ldots,Q_{n,p_m,q}$ and $Q_{n,r_1,q},\ldots,Q_{n,r_s,q}$ be elements of $\sM\left( n,q \right)$ with $Q_{n,p_1,q},\ldots,Q_{n,p_m,q}$ distinct and $Q_{n,r_1,q},\ldots,Q_{n,r_s,q}$ distinct. If
	\begin{eqnarray*}
		\sum_{i=1}^m a_i Q_{n,p_i,q} = \sum_{j=1}^s b_j Q_{n,r_j,q}
	\end{eqnarray*} 
	with $a_i>0,b_i> 0$ for all $i$ and $\sum_{i=1}^m a_i = \sum_{j=1}^m b_i= 1$, then $m=s$ and there exists some permutation $\sigma$ such that $a_i = b_{\sigma(i)}$ and $p_i = r_{\sigma(i)}$. 

\end{cor}
The proof is almost identical to the proof of Corollary \ref{cor:multident}, so we will omit it. Using these proof techniques one could establish additional identifiability/determinedness style results for multinomial mixture models along the lines of Theorems \ref{thm:liident}, \ref{thm:lidet}, and \ref{thm:ji}. Furthermore it seems likely that one could use the algorithm described in the next section or from \cite{anandkumar14,arora12, rabani14} to recover these components, using the transform $T_{n,q}$.
\section{Algorithms} \label{sec:algorithm}
Here we will present a few algorithms for the recovery of mixture components and proportions from data. The algorithms are quite general and can be applied to any measurable space. Unfortunately, due to the generality of the proposed algorithms, some of the implementation details are setting specific which makes in-depth theoretical analysis difficult. As one concrete illustration, we will show consistency for categorical measures.

Let $\sum_{i=1}^m w_i \delta_{\mu_i}$ be an arbitrary mixture of measures on some measurable space $(\Omega,\sF)$, which we are interested in recovering. Let $p_1,\ldots,p_m$ be square integrable densities with respect to a dominating measure $\xi$, with $\int_A p_i d\xi = \mu_i\left( A \right)$ for all $i \in [m]$ and $A \in \sF$. A measure $\xi$ and densities $p_1,\ldots,p_m$ satisfying these properties are guaranteed to exist as a consequence of Lemma \ref{lem:himbed}.

We will initially consider the situation where we have $2m$ samples per random group and have access to the tensors $\sum_{i=1}^m w_i p_i^{\otimes 2m}$ and $\sum_{i=1}^m w_i p_i^{\otimes 2m-2}$. In a finite discrete space, estimating these tensors is equivalent to estimating moment tensors of order $2m$ and $2m-2$. For measures over $\rn^d$ dominated by the Lebesgue measure, one could estimate these tensors using a kernel density estimator in $\rn^{d(2m)}$ and $\rn^{d(2m-2)}$ using each sample group as a kernel center. We will also assume that $p_1,\ldots,p_m$ have distinct norms. We will need to introduce tensor products of bounded linear operators. The following lemma is exactly proposition 2.6.12 from \cite{kadison83}.
\begin{lem} \label{lem:hsprod}
	Let $H_1,\ldots,H_n,H_1',\ldots,H_n'$ be Hilbert spaces and let $U_i \in \mathcal{L}\l(H_i,H_i'\r)$ for all $i \in \left[ n \right]$. There exists a unique 
	\begin{align*}
		U \in \mathcal{L}\left( H_1\otimes \cdots \otimes H_n,H_1'\otimes \cdots \otimes H_n' \right),
	\end{align*}
	such that $U\left( h_1\otimes \cdots \otimes h_n \right) = U_1\left( h_1 \right) \otimes \cdots \otimes U_n\left( h_n \right)$ for all $h_1 \in H_1, \ldots, h_n \in H_n$.
\end{lem}
\begin{defin}\label{def:prodop}
	The operator constructed in Lemma \ref{lem:hsprod} is called the {\em tensor product of $U_1,\ldots,U_n$} and is denoted $U_1\otimes \cdots \otimes U_n$.
\end{defin}
The following equality is mentioned in \cite{kadison83}.

\begin{lem} \label{lem:prodopbnd}
	Let $U_1,\ldots,U_n$ be defined as in Lemma \ref{lem:hsprod}. Then
	\begin{align*}
		\l\|U_1 \otimes \cdots \otimes U_n \r\| = \l\|U_1\r\| \l\|U_2\r\| \cdots \l\|U_n\r\|.
	\end{align*}
\end{lem}

Before we introduce the algorithms we will discuss an important point regarding computational implementation and Lemmas \ref{lem:hstens} and \ref{lem:hsprod}. For the remainder of this paragraph we will assume that Euclidean spaces are equipped with the standard inner product. Vectors in a space of tensor products of Euclidean space, for example $\rn^{d_1}\otimes \cdots \otimes \rn^{d_s}$ are easily represented on computers as elements of $\rn^{d_1\times \cdots \times d_s}$ \cite{symtensorrank}. Linear operators from some Euclidean tensor space to another can also be easily represented. Furthermore the transformation in Lemma \ref{lem:hstens} and the construction of new operators from Lemma \ref{lem:hsprod} can be implemented in computers by  ``unfolding'' the tensors into matrices, applying common linear algebraic manipulations and ``folding'' them back into tensors. The inner workings of these manipulations are beyond the scope of this paper and we refer the reader to \cite{golub1996} for details. Practically speaking this means the manipulations mentioned in Lemmas \ref{lem:hstens} and \ref{lem:hsprod} are straightforward to implement with a bit of tensor programming knowhow. Implementation may also be streamlined by using programming libraries that assist with these tensor manipulations such as the NumPy library for Python.

Because of the points mentioned in the previous paragraph, the following algorithm is readily implementable for estimating categorical distributions, where the measures can be represented as probability vectors on a Euclidean space. Similarly, we expect that these techniques could be extended to probability densities on Euclidean space using kernel density estimators with a kernel function with easily computable $L^2$ inner products (for example Gaussian kernels) although we suspect that implementation of such an algorithm may be significantly more involved.

To begin our analysis we will apply the transform from Lemma \ref{lem:hstens} to get the operator
\begin{eqnarray*}
	C = \sum_{i=1}^m w_i p_i^{\otimes m-1} \l<p_i^{\otimes m-1} ,\cdot\r> = \sum_{i=1}^m \sqrt{w_i} p_i^{\otimes m-1} \l<\sqrt{w_i}p_i^{\otimes m-1} ,\cdot\r>.
\end{eqnarray*}
\sloppy Here $C$ is a positive semi-definite (PSD) operator in  $\sL \l(L^2\left( \Omega,\sF,\xi \right)^{\otimes m-1}\r)$. Let $C^\dagger$ be the (Moore-Penrose) pseudoinverse of $C$ and $W = \sqrt{C^\dagger}$. Now $W$ is an operator that whitens $\sqrt{w_1} p_1^{\otimes m-1},\ldots, \sqrt{w_m}p_m^{\otimes m-1}$. That is, $W \sqrt{w_1} p_1^{\otimes m-1},\ldots,W \sqrt{w_m}p_m^{\otimes m-1}$ are orthonormal vectors. Using the operator construction from Lemma \ref{lem:hsprod} we can construct $I\otimes W\otimes I \otimes W$ where, for all simple tensors in $L^2\left( \Omega,\sF,\xi \right)^{\otimes 2m}$ we have,
\begin{align*}
	&(I\otimes W \otimes I \otimes W) \left( x_1\otimes \cdots \otimes x_{2m} \right) \\
	&\quad = x_1 \otimes W\left( x_2 \otimes \cdots \otimes x_m \right) \otimes x_{m+1} \otimes W\left(x_{m+2} \otimes \cdots \otimes x_{2m}  \right).
\end{align*}
Applying $I\otimes W \otimes I \otimes W$ to $\sum_{i=1}^{m} w_i p_i^{\otimes 2m}$ yields
\begin{eqnarray*}
	\sum_{i=1}^m w_i p_i\otimes W p_i^{\otimes m-1} \otimes p_i\otimes W p_i^{\otimes m-1},
\end{eqnarray*}
which can again be represented as a PSD operator
\begin{eqnarray*}
	S &\triangleq& \sum_{i=1}^m w_i p_i\otimes W p_i^{\otimes m-1} \l< p_i\otimes W p_i^{\otimes m-1},\cdot\r> \\
		 &=& \sum_{i=1}^m  p_i\otimes W \sqrt{w_i}p_i^{\otimes m-1} \l< p_i\otimes W \sqrt{w_i} p_i^{\otimes m-1},\cdot\r>.
\end{eqnarray*}
For $i\neq j$ it follows that $ p_i\otimes\sqrt{w_i} W p_i^{\otimes m-1} \perp  p_j\otimes W\sqrt{w_j} p_j^{\otimes m-1}$. To see this 
\begin{align*}
	&\l< p_i\otimes W \sqrt{w_i} p_i^{\otimes m-1} ,  p_j\otimes W \sqrt{w_j}p_j^{\otimes m-1}\r>\\
	&\qquad =\l<p_i,p_j\r> \l<  W \sqrt{w_i} p_i^{\otimes m-1} ,W\sqrt{w_j} p_j^{\otimes m-1}\r>\\
	       &\qquad =\l<p_i,p_j\r> 0 \\
		      &\qquad = 0.
\end{align*}
Also note that 
\begin{eqnarray*}
	\l\| p_i\otimes W \sqrt{w_i} p_i^{\otimes m-1} \r\|^2
	&=&\l< p_i\otimes W \sqrt{w_i} p_i^{\otimes m-1} ,p_i\otimes W \sqrt{w_i} p_i^{\otimes m-1} \r>\\
	&=&\l< p_i ,p_i \r>\l<W \sqrt{w_i} p_i^{\otimes m-1},W \sqrt{w_i} p_i^{\otimes m-1} \r> \\
	       &=&\l\|p_i\r\|^2.
\end{eqnarray*}
If $p_1,\ldots,p_m$ have distinct norms then it follows that 
\begin{align*}
    \sum_{i=1}^m  p_i\otimes W \sqrt{w_i}p_i^{\otimes m-1} \l< p_i\otimes W \sqrt{w_i}p_i^{\otimes m-1},\cdot\r>
\end{align*}
is the unique spectral decomposition of $S$ since the vectors $ p_1\otimes W \sqrt{w_1}p_1^{\otimes m-1},\ldots,  p_m\otimes W \sqrt{w_m}p_m^{\otimes m-1}$ are orthogonal, have distinct norms, and thus distinct positive eigenvalues. Given an eigenvector of $S$, $p_i\otimes W \sqrt{w_i} p_i^{\otimes m-1}$, we need only view it as a linear operator $p_i\l< W \sqrt{w_i} p_i^{\otimes m-1},\cdot\r>$ and apply this operator to some vector $z$ which is not orthogonal to $W \sqrt{w_i} p_i^{\otimes m-1}$, thus yielding $p_i$ scaled by $ \l<W \sqrt{w_i} p_i^{\otimes m-1},z\r>$.

Were the norms of $p_1,\ldots,p_m$ not distinct, then there would not be a spectral gap between some of the eigenvalues in $S$, and a spectral decomposition of $S$ may contain some eigenvectors that are not $ p_1\otimes W \sqrt{w_1}p_1^{\otimes m-1},\ldots,  p_m\otimes W \sqrt{w_m}p_m^{\otimes m-1}$, but are instead linear combinations of these vectors.

Once the mixture components $p_1,\ldots,p_m$ are recovered form the spectral decomposition we can estimate the mixture proportions. From these mixture components we can construct the tensors $p_1^{\otimes 2m-2}, \ldots , p_m^{\otimes 2m-2}$. These tensors are linearly independent by Lemma \ref{lem:linind}. The tensor $ \sum_{i=1}^m w_i p_i^{\otimes 2m-2}$ is known. By the linear independence of the components there is exactly one solution for $a_1,\ldots, a_m$ in the equation
\begin{eqnarray*}
	\sum_{i=1}^m w_i p_i^{\otimes 2m-2} = \sum_{j=1}^m a_j p_j^{\otimes 2m-2},
\end{eqnarray*}
so simply minimizing $\l\|\sum_{i=1}^m w_i p_i^{\otimes 2m-2} - \sum_{j=1}^m a_j p_j^{\otimes 2m-2}\r\|$ over $a_1,\ldots,a_m$ will give us the mixture proportions. We could also use a different tensor power $\l\|\sum_{i=1}^m w_i p_i^{\otimes r} - \sum_{j=1}^m a_j p_j^{\otimes r}\r\|$, so long as $r\ge m-1$ to guarantee independence of the components.

We can construct a similar algorithm with $4$ samples per group when the mixture components are known to be linearly independent. The details of this algorithm are in Appendix \ref{appx:spectral}. In such a setting it would be advisable to use the algorithms from \cite{anandkumar14,song14} since they better studied. We mention our algorithm for purely theoretical interest. There are likely a multitude of possible algorithms for the recovery of mixture components whose necessary group size depends on the geometry of the mixture components.

Taking inspiration from \cite{anandkumar14} and \cite{song14} we can suggest yet another algorithm.  The previous papers demonstrate algorithms for recovering mixture components which are measures on discrete spaces and $\rn^d$, from random groups of size 3, provided the mixture components are linearly independent. Given a mixture of measures $\sP = \sum_{i=1}^m w_i \delta_{\mu_i}$ with density functions $p_1,\ldots,p_m$, the tensors $p_1^{\otimes m-1}, \ldots , p_m^{\otimes m-1}$ are linearly independent. Thus, with $3m-3$ samples per random group, we can estimate the tensors $\sum_{i=1}^{m} w_i p_i^{\otimes 3m-3}$ and we can use the algorithms from the previous papers to recover $p_1^{\otimes m-1},\ldots,p_m^{\otimes m-1}$ from which it is straightforward to recover $p_1,\ldots,p_m$.

We can also recover the components with $2m-1$ samples per group. We will adopt the same setting as in our first algorithm, but with $2m-1$ samples per group in stead of $2m$. Let $W$ be as before. Using Lemma \ref{lem:hsprod} we can construct the operator $I\otimes W \otimes W$ on the space $L^2\left( \Omega,\sF,\xi \right)^{\otimes 2m-1}$ which maps simple tensors in the following way: $(I \otimes W \otimes W)\left( x_1\otimes \cdots \otimes x_{2m-1} \right) = x_1 \otimes W\left( x_2\otimes \cdots \otimes x_{m}  \right) \otimes W\left( x_{m+1}\otimes \cdots \otimes x_{2m-1} \right)$. Applying this operator to $\sum_{i=1}^m w_i p_i^{\otimes 2m-1}$ gives us the tensor

\begin{align*}
	A 
	&\triangleq \sum_{i=1}^m w_i p_i \otimes  W\l(p_i^{\otimes m-1}\r)\otimes W\l( p_{i}^{\otimes m-1}\r) \\
	&= \sum_{i=1}^m  p_i \otimes  W\l(\sqrt{w_i}p_i^{\otimes m-1}\r)\otimes W\l(\sqrt{w_i} p_{i}^{\otimes m-1}\r).
\end{align*}
From Lemma \ref{lem:hstens} we can transform the tensor $A$ into the operator $T$,
\begin{eqnarray}\label{eqn:T}
	T = \sum_{i=1}^m  p_i \otimes W\l( \sqrt{w_i} p_{i}^{\otimes m-1}\r)\l<W\l(\sqrt{w_i}p_i^{\otimes m-1}\r),\cdot\r>  .
\end{eqnarray}
Now the operator $TT^H$ is
\begin{align*}
	TT^H =  &\sum_{i=1}^m  p_i \otimes W\l( \sqrt{w_i} p_{i}^{\otimes m-1}\r)\bigg\langle W\l(\sqrt{w_i}p_i^{\otimes m-1}\r),\cdots \\
		       &\qquad \sum_{j=1}^m   W\l( \sqrt{w_j} p_{j}^{\otimes m-1}\r)\l<p_j \otimes W\l(\sqrt{w_j}p_j^{\otimes m-1}\r),\cdot\r>\bigg\rangle \\
	=  &\sum_{i=1}^m  p_i \otimes W\l( \sqrt{w_i} p_{i}^{\otimes m-1}\r)\l<p_i \otimes W\l(\sqrt{w_i}p_i^{\otimes m-1}\r),\cdot\r> 
\end{align*}
which is simply the operator $S$ from the previous section.
The last step is justified since the vectors $W\left( \sqrt{w_1} p_1^{\otimes m-1} \right),\ldots, W\left( \sqrt{w_m} p_m^{\otimes m-1} \right)$ are orthonormal.
This tensor is precisely the tensor from which we recovered the mixture components in the first algorithm.
\subsection{Spreading the Eigenvalue Gaps for Categorical Distributions} \label{sec:eigspread}
Here we will introduce a trick to guarantee that the norms of the mixture component distributions are distinct. Let $\left( \Omega, 2^\Omega \right)$ be a finite discrete measurable space with $\Omega = \l\{\omega_1,\ldots,\omega_d\r\}$. Let $\mu_1,\ldots,\mu_m$ be distinct measures on this space. Let $y_1,\ldots, y_d\simiid \text{unif}\left( 1,2 \right)$ and let $\xi$ be a random measure on $\left( \Omega,2^\Omega \right)$ defined by $\xi\l(\left\{ \omega_i \right\}\r) = y_i$ for all $i$. Clearly $\xi$ dominates all $\mu_1,\ldots,\mu_m$ and thus we can define Radon-Nikodym derivatives $p_i = \frac{d\mu_i}{d \xi}$ for all $i$. We will treat these Radon-Nikodym derivatives as being elements in $L^2\left( \Omega, 2^\Omega, \xi \right)$. We have the following lemma
\begin{lem} \label{lem:distinctnorm}
	With probability one
	\begin{eqnarray*}
		\int p_i(\omega)^2 d\xi(\omega) \neq \int p_j(\omega)^2 d\xi(\omega)
	\end{eqnarray*}
	for all $i\neq j$.
\end{lem}
\begin{proof}
	Observe that, for all $i,j$,
	\begin{eqnarray*}
		\int_{\left\{ \omega_j \right\}} p_i d\xi 
		= p_i(\omega_j)\xi\left( \left\{ \omega_j \right\} \right) 
		= p_i(\omega_j) y_j
		= \mu_i\left( \left\{ \omega_j \right\} \right)
	\end{eqnarray*}
	and thus $p_i\left( \omega_j \right) = \frac{\mu_i\left( \left\{ w_j \right\} \right)}{y_j}$. We will show that $\l\|p_1\r\|_{\lrd}^2 \neq \l\|p_2\r\|_{\lrd}^2$ with probability one, which implies $\l\|p_i\r\|_{\lrd} \neq \l\|p_j\r\|_{\lrd}$ for all $i\neq j$ with probability one (here and for the rest of the paper $\l\|\cdot \r\|_\lrd$ will denote the standard Euclidean norm on $\rn^d$ and $\l<\cdot,\cdot\r>_\lrd$ the standard inner product).

	Because $\mu_1 \neq \mu_2$ it follows that there exists some $j$ such that $\mu_1\left( \l\{\omega_j\r\} \right) \neq \mu_2\left( \left\{ \omega_j \right\} \right)$. Without loss of generality we will assume that $j=1$ in the previous statement. Now we have \small
	\begin{eqnarray*}
		&&P\left( \int p_1\left( \omega \right)^2 d\xi\left( \omega \right) = \int p_2\left( \omega \right)^2 d\xi\left( \omega \right) \right)\\
                &=& P\left( \sum_{i=1}^d \frac{\mu_1\left( \l\{\omega_i\r\} \right)^2}{y_i} = \sum_{j=1}^d \frac{\mu_2\left( \l\{\omega_j\r\} \right)^2}{y_j} \right)\\
                &=& P\left(\l( \frac{ \mu_1\left( \l\{\omega_1\r\} \right)^2}{y_1}-  \frac{ \mu_2\left( \l\{\omega_1\r\}^2 \right)}{y_1}\r) = \l(\sum_{i=2}^d \frac{\mu_1\left( \l\{\omega_i\r\} \right)^2}{y_i} - \sum_{j=2}^d \frac{\mu_2\left( \l\{\omega_j\r\} \right)^2}{y_j}\r) \right)\\
	\end{eqnarray*}
	\normalsize
        which is clearly zero since $\left( \mu_1\left( \l\{\omega_1\r\} \right) \right)^2 - \left( \mu_2\left( \l\{\omega_1\r\} \right) \right)^2 \neq 0$ and $y_1,\ldots, y_d$ are all independent random variables and from a non-atomic measure.
\end{proof}

Applying the previous trick with the recovery algorithm for groups of size $2m-1$ we have an algorithm for recovering mixtures on finite measure spaces with $m$ components. The paper \cite{rabani14} recovers the mixture components given a setting almost identical to ours, but we feel that our algorithm is more straightforward and easily extended to non-discrete spaces.

\subsection{Recovery Algorithm For Discrete Spaces}
Let $\left( \Omega, 2^\Omega \right)$ be a finite measurable space with $ \l| \Omega \r| = d$. To simplify exposition we will assume that $\Omega$ is simply the set of $d$ dimensional indicator vectors in $\rn^d$, $e_1,\ldots, e_d$. Note that Euclidean space with the standard inner product is $L^2\left( \Omega,2^{\Omega}, \sum_{i=1}^d \delta_{e_i} \right) = \lrd$. Let $\mu_1, \ldots, \mu_m$ be distinct probability measures on $\Omega$. Let $\sP = \sum_{i=1}^m w_i \delta_{\mu_i}$ be a mixture of measures. Let $\tilde{p}_i \triangleq \mathbb{E}_{x \sim \mu_i}\left[ x \right]$ for all $i$. Note that $\tilde{p}_{i,j} = \mu_i\left( \left\{ e_j \right\} \right)$ for all $i,j$. Let $\bX_1,\bX_2,\ldots \simiid V_{2m-1}\left( \sP \right)$ with $\bX_i = \left[ X_{i,1},\ldots, X_{ i,2m-1 } \right]$. 

To begin we construct the random dominating measure described in Section \ref{sec:eigspread}. Let $y_1,\ldots,y_d \simiid \text{unif}\left( 1,2 \right)$. The random dominating measure $\xi$ is defined by $\xi\left( \left\{ e_i \right\} \right) = y_i$ for all $i$. Let $p_i = \frac{d \mu_i}{d \xi}$, i.e. $p_i \left( e_{j} \right) =  \frac{\tilde{p}_{i,j}}{y_j}$ for all $i$ and $j$. There is a bit of a computational issue with this representation for the densities $p_1,\ldots, p_m$ since the new dominating measure changes the inner product from the standard inner product. We can remedy this with the following lemma.
\begin{lem}\label{lem:domtrans}
	Let $x,v \in \lrd$, $\xi$ be as above, and 
	\begin{eqnarray*}
		B =\l[	\begin{matrix}
				\frac{1}{\sqrt{y_1}} & 0& 0 & \cdots & 0\\
				0& \frac{1}{\sqrt{y_2}} & 0 &\cdots & 0 \\
				0&0&\ddots&&\vdots\\
				\vdots &\vdots&&\ddots& 0\\
				0& 0 & \cdots& 0 &  \frac{1}{\sqrt{y_d}}
		\end{matrix}\r].
	\end{eqnarray*}
	Then $ \l<Bx,Bv \r>_{L^2\left( \Omega,2^{\Omega},\xi \right)}  = \l<x,v\r>_\lrd$.
\end{lem}
\begin{proof}[Proof of Lemma \ref{lem:domtrans}]
	We have
	\begin{eqnarray*}
		\l<Bx,Bv \r>_{L^2\left( \Omega,2^{\Omega},\xi \right)} 
		&=& \int (Bx)(i) (Bv)(i) d\xi(i)\\
		&=& \sum_{i=1}^d (Bx)(i) (Bv)(i) y_i\\
		       &=& \sum_{i=1}^d \frac{x(i)}{\sqrt{y_i}} \frac{v(i)}{\sqrt{y_i}} y_i\\
			      &=& \sum_{i=1}^d x(i) y(i)\\
			      &=& \l<x,y\r>_\lrd.
	\end{eqnarray*}
\end{proof}

From this lemma we have that $B$, when considered as an operator in $\sL \l(\lrd, L^2\left( \Omega, 2^\Omega , \xi\right)\r)$, is a unitary transform. We are interested in estimating the tensor $\sum_{i=1}^m w_i p_i^{\otimes 2m-1}$, but in order to keep the algorithm operating in standard Euclidean space we will instead transform it into $\lrd$. To this end consider an arbitrary $i$. We have
\begin{eqnarray*}
	B^{-1}p_i 
	&=& B^{-1}\left[p_{i,1} ,\ldots, p_{i,d}  \right]^T\\
	&=& B^{-1}\left[\frac{\tilde{p}_{i,1}}{y_1} ,\ldots, \frac{\tilde{p}_{i,d}}{y_d}  \right]^T\\
	       &=&\left[\frac{\tilde{p}_{i,1}}{\sqrt{y_1}} ,\ldots, \frac{\tilde{p}_{i,d}}{\sqrt{y_d}}  \right]^T,
\end{eqnarray*}
and thus $B^{-1} p_j = B \tilde{p}_j$ for all $j$.

We will use the following lemma to find the expected value of 
\begin{align*}
	\mathbb{E}\left[ BX_{i,1}\otimes \cdots \otimes BX_{i,2m-1} \right]
\end{align*}
\begin{lem}\label{lem:tensind}
	Let $n>1$ and $Z_1, \ldots, Z_n$ be independent random vectors in $\rn^{d_1},\ldots, \rn^{d_n}$ such that $\mathbb{E} \left[ Z_i \right] $ exists for all $i$. Then $\mathbb{E}\left[ Z_1\otimes \cdots \otimes Z_n \right] = \mathbb{E}\left[ Z_1 \right] \otimes \cdots \otimes \mathbb{E}\left[ Z_n \right]$.
\end{lem}
\begin{proof}[Proof of Lemma \ref{lem:tensind}]
	Let $\l[i_1,\ldots, i_n\r] \in \rn^{d_1}\times \cdots \times \rn^{d_n}$ be arbitrary. We have that 
	\begin{eqnarray*}
		\mathbb{E}\left[ Z_1\otimes\cdots \otimes Z_n \right]_{i_1,\ldots,i_n}
		&=& \mathbb{E}\left[ Z_{1,i_1} \cdots Z_{n,i_n} \right]\\
		&=& \mathbb{E}\left[ Z_{1,i_1}\r] \cdots\mathbb{E} \l[ Z_{n,i_n} \right].
	\end{eqnarray*}
	Since $i_1,\ldots, i_n$ were arbitrary it implies that all entries of $\mathbb{E}\left[ Z_1\otimes \cdots \otimes Z_n \right]$ and $\mathbb{E}\left[ Z_1 \right]\otimes \cdots \otimes \mathbb{E}\left[ Z_n \right]$ are equal.
\end{proof}
Recall that $X_{i,1},\ldots,X_{i,2m-1} \simiid \mu$ with $\mu \sim \sP$. From the previous lemma and the definition of $\tilde{p}_i$ it follows that 
\begin{align*}
	&\mathbb{E}\left[ BX_{i,1} \otimes \cdots \otimes BX_{i,2m-1} \right]\\
	&=\mathbb{E}_{\mu\sim \sP}\left[ \mathbb{E}\left[ BX_{i,1} \otimes \cdots \otimes BX_{i,2m-1} \right| \mu] \right] \\
	&=\mathbb{E}_{\mu\sim \sP}\left[ \mathbb{E}\left[ BX_{i,1} |\mu\r] \otimes \cdots \otimes \mathbb{E}\l[BX_{i,2m-1}| \mu \right] \right] \\
	&=\mathbb{E}_{\mu\sim \sP}\left[ B\mathbb{E}\left[ X_{i,1}| \mu \r] \otimes \cdots \otimes B\mathbb{E}\l[X_{i,2m-1} | \mu\right] \right] \\
	&=\sum_{i=1}^m w_i B\mathbb{E}\left[ X_{i,1} |\mu = \mu_i\r] \otimes \cdots \otimes B\mathbb{E}\l[X_{i,2m-1}|\mu = \mu_i \right]  \\
	&=\sum_{i=1}^m w_i \left( B\tilde{p}_i \right)^{\otimes 2m-1} .
\end{align*}
Let $Y_{i,j} = B X_{i,j}$. Now we will construct the whitening operator. To do this first construct the operator 
\begin{eqnarray*}
	\widehat{C} =
	\begin{split}\frac{1}{\left( 2m-1 \right)!}\frac{1}{n} \sum_{i=1}^n \sum_{\sigma \in S_{2m-1}} Y_{i,\sigma\left( 1 \right)} \otimes \cdots \otimes Y_{i,\sigma\left( m-1 \right)}\\
	\l<Y_{i, \sigma \left( m \right)} \otimes \cdots \otimes Y_{i, \sigma \left( 2m-2 \right)}, \cdot \r>.\end{split}
\end{eqnarray*}
There are some repeated terms in the previous summation, which is not an issue. Instead we could have set $\widehat{C}$ to be equal to \small
\begin{eqnarray*}
	\frac{1}{\left( 2m-2 \right)!}\frac{1}{n} \sum_{i=1}^n \sum_{\sigma \in S_{2m-2}} Y_{i,\sigma\left( 1 \right)} \otimes \cdots \otimes Y_{i,\sigma\left( m-1 \right)}\l<Y_{i, \sigma \left( m \right)} \otimes \cdots \otimes Y_{i, \sigma \left( 2m-2 \right)}, \cdot \r>,
\end{eqnarray*} \normalsize
but this would not utilize all the data, specifically $Y_{1,2m-1},\ldots, Y_{n,2m-1}$. In the second operator the average over $S_{2m-2}$ functions as a projection onto the space of symmetric tensors and the summation over $S_{2m-1}$ in the definition of $\widehat{C}$ serves a similar purpose. Viewed alternatively, the distribution of $[Y_{i,1}, \ldots, Y_{i,2m-1}]^T$ does not change if we reorder the entries of the vector, so the summation is considering all possible orderings of random groups. This symmetrization conveniently assures that $\widehat{C}$ is a Hermitian operator.
This $\widehat{C}$ is estimating the $C$ mentioned in the algorithm. Let $\lambda_{\widehat{C},1}, \ldots, \lambda_{\widehat{C},m}$ be the top $m$ eigenvalues of $\widehat{C}$ and $v_{\widehat{C},1},\ldots,v_{\widehat{C},m}$ be their associated eigenvectors. We can now construct the whitening operator
\begin{eqnarray*}
	\widehat{W} = \sum_{i=1}^m \lambda_{\widehat{C},i}^{-\frac{1}{2}} v_{\widehat{C},i} \l<v_{\widehat{C},i}, \cdot\r>.
\end{eqnarray*}
Now construct the tensor
\begin{eqnarray*}
	\widehat{A} = 
	\begin{split}
		\frac{1}{\left( 2m-1 \right)!} \frac{1}{n} \sum_{i=1}^n \sum_{\sigma \in S_{2m-1}} Y_{i,\sigma\left( 1 \right)} \otimes \widehat{W} \left( Y_{i,\sigma\left( 2 \right)} \otimes \cdots \otimes Y_{i,\sigma\left( m \right)} \right) \otimes \cdots\\
		\widehat{W}\left( Y_{i,\sigma\left( m+1 \right)} \otimes \cdots \otimes Y_{i,\sigma\left( 2m-1 \right)} \right).
	\end{split}
\end{eqnarray*}
Using simple unfolding techniques we can transform $\widehat{A}$ in to the operator $\widehat{T}$: 
\begin{eqnarray*}\begin{split}
	\widehat{T} = \frac{1}{\left( 2m-1 \right)!} \frac{1}{n} \sum_{i=1}^n \sum_{\sigma \in S_{2m-1}} Y_{i,\sigma\left( 1 \right)} \otimes \widehat{W} \left( Y_{i,\sigma\left( 2 \right)} \otimes \cdots \otimes Y_{i,\sigma\left( m \right)} \right) \cdots\\
	\l< \widehat{W}\left( Y_{i,\sigma\left( m+1 \right)} \otimes \cdots \otimes Y_{i,\sigma\left( 2m-1 \right)} \right), \cdot \r>,
\end{split}
	\end{eqnarray*}
	as well as its Hermitian, $\widehat{T}^H$:
	\begin{eqnarray*}
	\begin{split}\frac{1}{\left( 2m-1 \right)!} \frac{1}{n} \sum_{i=1}^n \sum_{\sigma \in S_{2m-1}} \widehat{W}\left( Y_{i,\sigma\left( m+1 \right)} \otimes \cdots \otimes Y_{i,\sigma\left( 2m-1 \right)} \right)\cdots \\ \l< Y_{i,\sigma\left( 1 \right)} \otimes \widehat{W} \left( Y_{i,\sigma\left( 2 \right)} \otimes \cdots \otimes Y_{i,\sigma\left( m \right)} \right), \cdot \r>.\end{split}
	\end{eqnarray*}
	Let $v_1,\ldots, v_m$ be the top $m$ eigenvectors of $\widehat{T}\widehat{T}^H$  (\ref{eqn:T}), which will be elements of $\lrd^{\otimes m}$.
	\sloppy These vectors are estimates of $ \l\|B\tilde{p}_1\r\|_2^{-1}B \tilde{p}_1 \otimes \widehat{W} \sqrt{w_1} \left( B\tilde{p}_1 \right)^{\otimes m-1},\ldots,\l\|B\tilde{p}_m\r\|_2^{-1}B \tilde{p}_m \otimes \widehat{W} \sqrt{w_m} \left( B\tilde{p}_m \right)^{\otimes m-1}$ (possibly multiplied by $-1$). The factors in front of the tensors normalize the tensors to have norm 1.

	Using a transform of the form in Lemma \ref{lem:hstens}, we can implement a transform 
	\begin{align*}
		U: \lrd^{\otimes m} \to \hs \left( \lrd^{\otimes m-1}, \lrd  \right)
	\end{align*}
	which maps simple tensors $x_1\otimes \cdots \otimes x_m$ to $x_1 \l<x_2 \otimes \cdots \otimes x_m, \cdot \r>$. Applying this transform to $v_1,\ldots,v_m$ yields estimates of $\l\|B\tilde{p}_i\r\|_\lrd^{-1}B \tilde{p}_i \l< \widehat{W} \sqrt{w_i} \left( B\tilde{p}_i \right)^{\otimes m-1},\cdot\r>$, for all $i$.
        At this point one simply needs to find vectors $q_1,\ldots,q_m$ which are not orthogonal to $\widehat{W} \sqrt{w_1} \left( B\tilde{p}_1 \right)^{\otimes m-1},\ldots,\widehat{W} \sqrt{w_m} \left( B\tilde{p}_m \right)^{\otimes m-1}$ to get $\l\|B\tilde{p}_i\r\|_\lrd^{-1}B \tilde{p}_i \l< \widehat{W} \sqrt{w_i} \left( B\tilde{p}_i \right)^{\otimes m-1},q_i\r>$, which is $B\tilde{p}_i,\ldots, B\tilde{p}_i$ up to scaling. Such vectors can be found by simply using a tensor populated by iid standard normal random variables. After this we can recover $\tilde{p}_1,\ldots,\tilde{p}_m$, up to scaling, by simply applying $B^{-1}$, which we would then want to normalize to sum to one. Alternatively we could take the largest left singular vector of these operators. We will call these estimates $\widehat{p}_1,\ldots,\widehat{p}_m$.

	Using the data we can estimate the tensor 
        $\sum_{i=1}^m w_i \tilde{p}_i^{\otimes m-1}$ with the estimator
	\begin{eqnarray*}
		\widehat{E} = \frac{1}{2m-1} \frac{1}{n} \sum_{i=1}^n \sum_{\sigma \in S_{2m-1}} X_{i,\sigma\left( 1 \right)} \otimes \cdots \otimes X_{i,\sigma\left( m-1 \right)}
	\end{eqnarray*}

	To estimate the mixture proportions we find the value of $\alpha = \left( \alpha_1,\ldots, \alpha_m \right)$ which minimizes
	\begin{eqnarray*}
		\l\|\widehat{E} - \sum_{i=1}^m \alpha_i \widehat{p_i}^{\otimes m-1}\r\|.
	\end{eqnarray*}

	\subsection{Consistency of Recovery Algorithm}
	We will now show that the recovery algorithm for categorical distributions is consistent. Let $C,\widehat{C},T, \widehat{T},W,$ and $\widehat{W}$ be as they were defined in the first part of this section. The crux of our algorithm is the recovery of the eigenvectors of $TT^H$, from which we then recover the mixture components through the application of linear and continuous transforms to the eigenvectors. In order to simplify the notation in our explanation we will assume that the norms of $\tilde{p}_1,\ldots, \tilde{p}_m$ are distinct. We do this so that there are gaps in the spectral decomposition of $TT^H$ thus making the random dominating measure trick unnecessary. Were this not the case, we could simply represent the probability vectors as densities with respect to some dominating measure which makes their norms distinct, as we did in the previous section. Because of this assumption we can simply set $B$ to be the identity operator. From this we have that $p_i = \tilde{p}_i$ for all $i$ and $X_{i,j} = Y_{i,j}$ for all $i$ and $j$. The following theorem demonstrates that the algorithm does indeed recover the eigenvectors of $TT^H$.
	\begin{thm}\label{thm:consistency}
		With $T$ and $\widehat{T}$ defined as above, as $n\to \infty$ then
		\begin{align*}
			\l\|TT^H - \widehat{T}\widehat{T}^H\r\|_\hs \cip 0.
		\end{align*}
	\end{thm}
	\begin{proof}[Proof of Theorem \ref{thm:consistency}]
	Let 
	\begin{align*}
		Q &=  \sum_{i=1}^m w_i p_i^{\otimes 2m-1}
	\end{align*}
	and
	\begin{align*}
		\widehat{Q} &=   \frac{1}{\left( 2m-1 \right)!} \frac{1}{n} \sum_{i=1}^n \sum_{\sigma \in S_{2m-1}} X_{i,\sigma\left( 1 \right)} \otimes \cdots \otimes X_{i,\sigma\left( 2m-1 \right)}.
	\end{align*}
	Note that 
	\begin{eqnarray*}
		\left( I\otimes W \otimes W \right)\left( Q \right) &=&  \sum_{i=1}^m w_i p_i \otimes W\left( p_i^{\otimes m-1}  \right) \otimes W\left( p_i^{\otimes m-1} \right)
	\end{eqnarray*}
	and
	\begin{align*}
		&\left( I\otimes \widehat{W}\otimes \widehat{W} \right)(\widehat{Q})  \\
		&= \frac{1}{\left( 2m-1 \right)!} \frac{1}{n} \sum_{i=1}^n \sum_{\sigma \in S_{2m-1}} X_{i,\sigma\left( 1 \right)} \otimes \widehat{W} \left( X_{i,\sigma\left( 2 \right)} \otimes \cdots \otimes X_{i,\sigma\left( m \right)} \right) \otimes \cdots \\
		&\qquad \widehat{W}  \l(X_{i,\sigma\left( m+1 \right)} \otimes \cdots \otimes X_{i,\sigma\left( 2m-1 \right)}\r).
	\end{align*}
	Since the transform in Lemma \ref{lem:hstens} is unitary, we have that 
	\begin{eqnarray*}
            \l\|T - \widehat{T} \r\|_\hs = \l\|\left( I\otimes W\otimes W \right) (Q)  - \left( I\otimes \widehat{W}\otimes \widehat{W} \right)\left( \widehat{Q} \right)\r\|_{\lrd^{\otimes 2m-1}}.
	\end{eqnarray*}
	We will now show that $\l\|T - \widehat{T}\r\| \cip 0$. 
	\begin{eqnarray*}
		\l\|T - \widehat{T} \r\| 
		&\le& \l\|T - \widehat{T} \r\|_{\hs} \\
                &=& \l\|(I \otimes W \otimes W)(Q) - \left( I \otimes \widehat{W} \otimes \widehat{W}\right) \left( \widehat{Q} \right) \r\|_{\lrd^{\otimes 2m-1}}\\
                &\le& \l\|(I \otimes W \otimes W)(Q) -  (I \otimes W \otimes W)(\widehat{Q})\r\|_{\lrd^{\otimes 2m-1}}  \\
                && +\l\|  (I \otimes W \otimes W)(\widehat{Q}) - \left( I \otimes \widehat{W} \otimes \widehat{W}\right) \left( \widehat{Q} \right) \r\|_{\lrd^{\otimes 2m-1}}\\
                &\le& \l\|I \otimes W \otimes W\r\| \l\|Q - \widehat{Q}\r\|_{\lrd^{\otimes 2m-1}}  \\
                && +\l\|  I \otimes W \otimes W -  I \otimes \widehat{W} \otimes \widehat{W}  \r\| \l\|\widehat{Q}\r\|_{\lrd^{\otimes 2m-1}}.
	\end{eqnarray*}
	We have that $\mathbb{E} \left[ \widehat{Q} \right] = Q$ so the first summand goes to zero in probability by the law of large numbers. All we need to show is that $\l\|  I \otimes W \otimes W -  I \otimes \widehat{W} \otimes \widehat{W}  \r\| \cip 0$.

	From Lemma \ref{lem:prodopbnd} we have that 
	\begin{align*}
		\l\|  I \otimes W \otimes W -  I \otimes \widehat{W} \otimes \widehat{W}  \r\| 
		&\le \l\|I \r\| \l\|   W \otimes W -   \widehat{W} \otimes \widehat{W}  \r\| \\
		&= \l\|   W \otimes W -   \widehat{W} \otimes \widehat{W}  \r\|\\
		&\le \l\|   W \otimes W - W \otimes \widehat{W}\r\|+\cdots\\
		&\qquad \l\|W \otimes \widehat{W} -  \widehat{W} \otimes \widehat{W}  \r\|\\
		&= \l\|W\r\| \l\| W - \widehat{W}\r\| +\l\|\widehat{W} \r\| \l\|W  -  \widehat{W}   \r\|\\
		&= \l(\l\|W\r\| + \l\|\widehat{W} \r\| \r)\l\| W - \widehat{W}\r\|.
	\end{align*}
	The left factor converges in probability to $2\l\|W\r\|$ and the right factor converges to 0 in probability and so we have that $\l\|T - \widehat{T}\r\| \cip 0 $. From this we also have that $\l\| \widehat{T}\widehat{T}^H - TT^H\r\| \cip 0$. 
\end{proof}

As demonstrated earlier in this section the mixture components are recovered by applying a composition of linear and continuous operators to the eigenvectors of $TT^H$, thus consistent estimation of the eigenvectors of $TT^H$ gives us consistent estimation of the mixture components.

\section{Experiments}\label{sec:experiments}
	Here we will present some experimental results of our algorithm applied to a simple synthetic dataset. The sample space for the experiments is $\Omega = \left\{ 0,1,2 \right\}$. The mixture components of our dataset are $\mu_1,\mu_2,\mu_3$ with $\mu_1$ distributed according to a binomial distribution with $n=2$ and $p = 0.2$, $\mu_2$ is similar with $p=0.8$ and $\mu_3 = \frac{1}{3} \mu_1 + \frac{2}{3}\mu_2$. The component weights are $w_1 = 0.5, w_2 = 0.3, w_3 = 0.2$. We chose these mixture components so that they are not particularly nice. Specifically,  the mixture components are not linearly independent, and when considered as vectors in $\rn^3$, $\mu_1$ and $\mu_2$ have the same norm. Our mixture of measures is $\sP = \sum_{i=1}^3 w_i \delta_{\mu_i}$ and our samples come from $V_5\left( \sP \right)$.

	 We construct our own performance measure which allows us to judge the performance of the estimated components jointly. Let $\widehat{\mu}_1, \widehat{\mu}_2,\widehat{\mu}_3$ be the three estimates for the mixture components from some algorithm. We will view these estimates as vectors in $\rn^3$. Our performance measure is $\min_{\sigma \in S_3} \frac{1}{3}\sum_{i=1}^3 \l\|\mu_i - \widehat{\mu}_{\sigma\left( i \right)}\r\|_{\ell^1\left( \rn^3 \right)}$. That is, we take the average of total variations of the best matching of the estimated mixture components to the true components.
	\subsection{Proposed Algorithms}
        We include two different implementations of our proposed algorithm. For our first implementation we use the ``random dominating measure'' technique described in Section \ref{sec:eigspread}. The random dominating measure was generated using the square of iid Gaussian random variables with mean $0$ and standard deviation 0.03. We used the Gaussian random variables instead of a uniform distribution for the random dominating measure because the Gaussian random measure performed better. 

         The purpose of the random dominating measure is to create a spectral gap between the mixture components. Intuitively, it seems reasonable that if we choose the dominating measure ``well'' then we will end up with large spectral gaps without making any of the component norms so diminutive as to become unnoticeable. In the interest of exploring this idea we tested different dominating measures until we found one that improved algorithmic performance significantly and include these experimental results as well. The dominating measure we settled on for $\xi$ is $\xi\left( \left\{ 0 \right\} \right) = 3^2, \xi\left( \left\{ 1 \right\} \right) = 2^2$ and $\xi\left( \left\{ 2 \right\} \right) = 1$. We include the experimental results for this ``fixed dominating measure'' implementation. These experiments strongly indicate the possibility for significant improvements to our algorithm by choosing the dominating measure intelligently.

         Finally we made one minor adjustment to the algorithms described earlier. If the estimators described above yield a component which has a negative entry, we simply set the negative entry to zero and renormalise. 

	Both of these implementations were run on two experimental scenarios, one with 50,000 random groups and the other with 10,000,000 random groups. We repeated each experiment 20 times and report relevant statistics.

	\subsection{Competing Algorithms}
        We compare our algorithm to the algorithm from \cite{anandkumar14} as well as simply choosing $3$ measures uniformly at random from the probabilistic simplex. The randomly selected components algorithm was repeated 1000 times. The algorithm in \cite{anandkumar14} is designed to work on random groups with three samples and a mixture of measures with linearly independent components. We apply the algorithm in \cite{anandkumar14} to the population tensor associated with $V_3\left(\sP \right)$, not a finite sample of that tensor.
	\subsection{Results}
		The results are summarized in Table 1. As expected the algorithm from \cite{anandkumar14} is not capable of recovering the mixture components since they are not linearly independent. The algorithm in \cite{anandkumar14} uses a ``tensor power method'' to recover the mixture components. In that paper the authors demonstrate that this method is guaranteed to recover the components if they are linearly independent. In our experiments we noticed that the components returned from the tensor power method were not unique and depended on the vector chosen for the initialization of the algorithm. In our experiments we chose the initial vector randomly using an isotropic Gaussian distribution. We performed the tensor power method with many random initializations and the performance measure of the returned components always settled on one of two values, which are both reported in Table 1. Presumably this behaviour is also due to the violation of the linear independence assumption.
	\begin{table*}[!tb]
		\caption{Experimental Results}
		\hfill{}
		\begin{center}
			\begin{tabular}{| *{2}{c|}}
				\hline
				Method &Performance\\
				\hline
				Random Dominating Measure, 50,000 samples& Mean:0.1407, Variance:0.0169\\
				\hline
				Fixed Dominating Measure, 50,000 samples & Mean:0.0524, Variance:0.0011\\
				\hline
				Random Dominating Measure, 10,000,000 samples & Mean:0.0433, Variance:0.0062\\
				\hline
                                Fixed Dominating Measure, 10,000,000 samples & Mean:0.0037, Variance: $4e{-6}$\\
				\hline
				Randomly Selected Measures& Mean:0.5323, Variance:0.0203\\
				\hline
				Anandkumar, et al. \cite{anandkumar14}& 0.3214 or 0.1758\\
				\hline
			\end{tabular}
		\end{center}
		\hfill{}
	\end{table*}

        \section{Discussion}\label{sec:discussion}
	In closing, we offer the following observations related to our results.
	\subsection{Possible Algorithm Improvements}
	We feel that there is significant room left for improving our proposed algorithm. Though we do not include these experiments, we observed a phenomena that having a large separation between the norms of the components significantly improves the ability for the algorithm to recover the mixture components. As the experiments demonstrate, choosing a good dominating measure which separates the norms can improve performance. An avenue for possible improvement is intelligent selection of a dominating measure. One possible disadvantage of choosing the dominating measure with iid random variables is that a sort of central limit type of effect occurs which draws the norms together. Perhaps there is some way to select the dominating measure from the data which will improve performance.

	A second improvement may come from better estimates of the $C$ and $T$ operators in the algorithm. Principally, estimating these depends on good estimates of symmetric tensors which represent categorical distributions. It has been shown that the estimation of discrete distributions can be improved by not simply using the frequencies of each occurrence of each category \cite{lehmann03,valiant16,orlitsky15,kamath15,han15,paninski05}. It seems possible that leveraging the techniques used for estimating categorical distributions with the structure of symmetric tensors can yield improved estimates of the symmetric tensors we use and thus improve the performance of the algorithm.

	\subsection{Potential Statistical Test and Estimator}
	The results on determinedness suggest the possibility of a goodness of fit test. Suppose we have grouped samples from some mixture of measures $\sP' = \sum_{i=1}^{m'} w_i' \delta_{\mu_i'}$. Further suppose some null hypothesis 
	\begin{align*}
            H_0: \sP' = \sP \triangleq \sum_{i=1}^m w_i \delta_{\mu_i}.
	\end{align*}
	Given data from $V_{2m}\left( \sP' \right)$ we may be able to reject the null hypothesis provided we have some way of estimating $M \triangleq \sum_{i=1}^m w_i \mu_i^{\times 2m}$ from the groups of samples. We will call such an estimator $\widehat{M}$. If $\widehat{M}$ does not converge to $M$ then we can reject the null hypothesis. The implementation and analysis of such an estimator would depend on the setting and is outside the scope of this paper

	One interesting observation from the proof of Theorem \ref{thm:det} is that, if $\sP= \sum_{i=1}^{m}  w_i \delta_{\mu_i}$ is a mixture of measures, $p_i$ is a pdf for $\mu_i$ for all $i$, and $n>m$, then the rank of $\sum_{i=1}^m a_i p_i^{\otimes n} \otimes p_i^{\otimes n}$ will be exactly $m$. This suggests a statistical estimator for the number of mixture components. The form of this tensor is amenable to spectral methods since it is a positive semi-definite tensor of order 2, which is akin to a positive semi-definite matrix. Embedding the data with the kernel mean mapping, using a universal kernel \cite{micchelli06}, seems like a promising approach to constructing such a test or estimator.
	\subsection{Identifiability and the Value $2n-1$}
	The value $2n-1$ seems to carry some significance for identifiability beyond the setting we proposed. This value can also be found in results concerning metrics on trees \cite{pachter04}, hidden Markov models \cite{paz71}, and frame theory, with applications to signal processing \cite{balan06}. All of these results are related to identifiability of an object or the injectivity of an operator. We can offer no further insight as to why this value recurs, but it appears to be an algebraic phenomenon.
	\appendix 
	\section{Additional Proofs} \label{appx:proofs}
	Some of the proofs use Hilbert-Schmidt operators. See Definition \ref{def:hs} for the definition of Hilbert-Schmidt operator.
	\begin{proof}[Proof of Lemma \ref{lem:represent}]
		Because both representations are minimal it follows that $\alpha'_i \neq 0$ for all $i$ and $\nu_i' \neq \nu_j'$ for all $i \neq j$. From this we know $\sQ\left( \l\{\nu_i'\r\} \right) \neq 0$ for all $i$. Because $\sQ\left( \l\{\nu_i'\r\} \right) \neq 0$ for all $i$ it follows that for any $i$ there exists some $j$ such that $\nu_i' = \nu_j$. Let $\psi: \left[ r \right] \to \left[ r \right]$ be a function satisfying $\nu_i' = \nu_{\psi\left( i \right)}$. Because the elements $\nu_1,\ldots,\nu_r$ are also distinct, $\psi$ must be injective and thus a permutation. Again from this distinctness we get that, for all $i$, $\sQ\left( \left\{ \nu_i' \right\}  \right)= \alpha'_i =\alpha_{\psi\left( i \right)}$ and we are done.
	\end{proof}

	\begin{proof}[Proof of Lemma \ref{lem:ident} and \ref{lem:det}]
		We will proceed by contradiction. Let $\sP = \sum_{i=1}^m a_i \delta_{\mu_i}$ be $n$-identifiable/determined, let $\sP' = \sum_{j=1}^l b_j \delta_{\nu_j}$ be a different mixture of measures, with $l\le m$ for the $n$-identifiable case, and 
		\begin{eqnarray*}
			\sum_{i=1}^m a_i \mu_i^{\times q} = \sum_{j=1}^l b_j \nu_j^{\times q}
		\end{eqnarray*}
		for some $q>n$. Let $A \in \sF^{\times n}$ be arbitrary. We have
		\begin{eqnarray*}
			\sum_{i=1}^m a_i \mu_i^{\times q} &=& \sum_{j=1}^l b_j \nu_j^{\times q}\\
			\Rightarrow \sum_{i=1}^m a_i \mu_i^{\times q}\left( A\times \Omega^{\times q-n} \right) &=& \sum_{j=1}^l b_j \nu_j^{\times q}\left( A\times \Omega^{\times q-n} \right)\\
			\Rightarrow \sum_{i=1}^m a_i \mu_i^{\times n}\left( A \right) &=& \sum_{j=1}^l b_j \nu_j^{\times n}\left( A  \right).
		\end{eqnarray*}
		This implies that $\sP$ is not $n$-identifiable/determined, a contradiction.
	\end{proof}
	\begin{proof}[Proof of Lemma \ref{lem:noident} and \ref{lem:nodet}]
		Let a mixture of measures $\sP = \sum_{i=1}^m a_i \delta_{\mu_i}$ not be $n$-identifiable/determined. It follows that there exists a different mixture of measures $\sP' = \sum_{j=1}^l b_j \delta_{\nu_j}$, with $l\le m$ for the $n$-identifiability case, such that
		\begin{eqnarray*}
			\sum_{i=1}^m a_i \mu_i^{\times n} &=& \sum_{j=1}^l b_j \nu_j^{\times n}.
		\end{eqnarray*}
		Let $A \in \sF^{\times q}$ be arbitrary, we have
		\begin{eqnarray*}
			\sum_{i=1}^m a_i \mu_i^{\times n}\left( A\times \Omega^{\times n-q} \right) &=& \sum_{j=1}^l b_j \nu_j^{\times n}\left( A\times \Omega^{\times n-q} \right)\\
			\Rightarrow \sum_{i=1}^m a_i \mu_i^{\times q}\left( A  \right) &=& \sum_{j=1}^l b_j \nu_j^{\times q}\left( A \right)
		\end{eqnarray*}
		and therefore $\sP$ is not $q$-identifiable/determined.
	\end{proof}

	\begin{proof}[Proof of Lemma \ref{lem:l2prod}]
		Example 2.6.11 in \cite{kadison83} states that for any two $\sigma$-finite measure spaces $\left( S,\mathscr{S}, m \right), \left( S',\mathscr{S}', m' \right)$ there exists a unitary operator $U: L^2\left( S,\mathscr{S}, m \right) \otimes L^2 \left( S',\mathscr{S'}, m' \right) \to L^2\left( S\times S', \mathscr{S}\times \mathscr{S'}, m\times m' \right)$ such that, for all $f,g$,
		\begin{eqnarray*}
			U(f\otimes g) = f(\cdot)g(\cdot).
		\end{eqnarray*}
		Because $\left( \Psi, \sG, \gamma \right)$ is a $\sigma$-finite measure space it follows that $\left( \Psi^{\times m}, \sG^{\times m}, \gamma^{\times m} \right)$ is a $\sigma$-finite measure space for all $m\in \mathbb{N}$. We will now proceed by induction. Clearly the lemma holds for $n=1$. Suppose the lemma holds for $n-1$. From the induction hypothesis we know that there exists a unitary transform $U_{n-1}: L^2\left( \Psi, \sG, \gamma \right)^{\otimes n-1} \to L^2 \left( \Psi^{\times n-1} ,\sG ^{\times n-1}  , \gamma^{\times n-1} \right)$ such that for all simple tensors $f_1\otimes\cdots \otimes f_{n-1} \in L^2\left( \Psi, \sG, \gamma \right)^{\otimes n-1}$ we have $U_{n-1}\left( f_1\otimes\cdots \otimes f_{n-1} \right) = f_1(\cdot)\cdots f_{n-1}\left( \cdot \right)$. Combining $U_{n-1}$ with the identity map via Lemma \ref{lem:unitprod} we can construct a unitary operator $T_n: L^2\left( \Psi, \sG, \gamma \right)^{\otimes n-1} \otimes L^2\left( \Psi, \sG, \gamma \right) \to L^2 \left( \Psi^{\times n-1} ,\sG ^{\times n-1}  , \gamma^{\times n-1} \right) \otimes L^2\left( \Psi, \sG, \gamma \right)$, which maps $f_1\otimes\cdots\otimes f_{n-1}  \otimes f_n \mapsto f_1(\cdot)\cdots f_{n-1}(\cdot) \otimes f_n$.

		From the aforementioned example there exists a unitary transform $K_n:L^2\left( \Psi^{\times n-1},\sG^{\times n-1}, \gamma^{\times n-1} \right)\otimes L^2\left( \Psi,\sG, \gamma \right)\to L^2 \left( \Psi^{\times n-1} \times \Psi,\sG ^{\times n-1} \times \sG , \gamma^{\times n-1}\times \gamma\right)$ which maps simple tensors $g\otimes g' \in L^2\left( \Psi^{\times n-1},\sG^{\times n-1}, \gamma^{\times n-1} \right)\otimes L^2\left( \Psi,\sG, \gamma \right)$ as $K_n\left( g\otimes g' \right) = g(\cdot) g'(\cdot)$. Defining $U_n(\cdot)= K_n\left( T_n \left( \cdot \right) \right)$ yields our desired unitary transform.
	\end{proof}

	\begin{proof}[Proof of Lemma \ref{lem:unitprod}]
		Lemma \ref{lem:hsprod} states that there exists a continuous linear operator $\tilde{U}:H_1 \otimes \cdots \otimes H_n \to H_1' \otimes \cdots \otimes H_n'$ such that $\tilde{U}\left( h_1 \otimes\cdots \otimes h_n \right) = U_1(h_1) \otimes \cdots \otimes U_n(h_n)$ for all $h_1 \in H_1 ,\cdots, h_n \in H_n$. Let $\widehat{H}$ be the set of simple tensors in $H_1 \otimes \cdots \otimes H_n$ and $\widehat{H}'$ be the set of simple tensors in $H_1'\otimes \cdots \otimes H_n'$. Because $U_i$ is surjective for all $i$, clearly $\tilde{U}(\widehat{H}) = \widehat{H}'$. The linearity of $\tilde{U}$ implies that $\tilde{U}(\spn(\widehat{H}))= \spn(\widehat{H}')$. Because $\spn(\widehat{H}')$ is dense in $H_1'\otimes \cdots \otimes H_n'$ the continuity of $\tilde{U}$ implies that $\tilde{U}(H_1\otimes\cdots \otimes H_n) = H_1'\otimes \cdots \otimes H_n'$ so $\tilde{U}$ is surjective. All that remains to be shown is that $\tilde{U}$ preserves the inner product (see Theorem 4.18 in \cite{introhilb}). By the continuity of inner product we need only show that $\l<h, g\r>=\l<\tilde{U}(h), \tilde{U}(g)\r>$ for $h,g \in \spn(\widehat{H})$. With this in mind let $h_1,\ldots, h_N,g_1,\ldots,g_M$ be simple tensors in $H_1\otimes \cdots \otimes H_n$. We have the following
		\begin{eqnarray*}
			\l<\tilde{U}\l(\sum_{i=1}^N h_i\r),\tilde{U}\l(\sum_{j=1}^M g_j\r) \r>
			&=& \l<\sum_{i=1}^N \tilde{U}\l(h_i\r),\sum_{j=1}^M \tilde{U}\l(g_j\r) \r>\\
		 &=& \sum_{i=1}^N\sum_{j=1}^M\l< \tilde{U}\l(h_i\r), \tilde{U}\l(g_j\r) \r>\\
		 &=& \sum_{i=1}^N\sum_{j=1}^M\l< h_i, g_j \r>\\
		 &=& \l< \sum_{i=1}^Nh_i, \sum_{j=1}^M g_j \r>.
		\end{eqnarray*}
		We have now shown that $\tilde{U}$ is unitary which completes our proof.
	\end{proof}
	\begin{proof}[Proof of Lemma \ref{lem:linind}]
		We will proceed by induction. For $n=2$ the lemma clearly holds. Suppose the lemma holds for $n-1$ and let $h_1,\ldots, h_n$ satisfy the assumptions in the lemma statement. Let $\alpha_1,\ldots, \alpha_n$ satisfy
		\begin{eqnarray}
			\sum_{i=1}^n \alpha_i h_i^{\otimes n-1}  = 0. \label{lisum}
		\end{eqnarray}
		To finish the proof we will show that $\alpha_1$ must be zero which can be generalized to any $\alpha_i$.
		Applying Lemma \ref{lem:hstens} to (\ref{lisum}) we get
		\begin{eqnarray}
			\sum_{i=1}^n \alpha_i h_i^{\otimes n-2}  \l<h_i, \cdot \r>  = 0. \label{lioper}
		\end{eqnarray}
		Because $h_1$ and $h_n$ are linearly independent we can choose $z$ such that $\l<h_1,z\r> \neq 0$ and $z\perp h_n$. Plugging $z$ into (\ref{lioper}) yields
		\begin{eqnarray*}
			\sum_{i=1}^{n-1} \alpha_i h_i^{\otimes n-2}\l<h_i, z \r>  = 0 
		\end{eqnarray*}
		and therefore $\alpha_1=0$ by the inductive hypothesis.
	\end{proof}
	\begin{proof}[Proof of Lemma \ref{lem:tensrank}]
		Let $\dim\left( \spn\left( h_1,\ldots,h_m \right) \right)= l$ and let $h = \sum_{i=1}^m h_i^{\otimes 2}$. Without loss of generality assume that $h_1,\ldots,h_l$ are linearly independent and nonzero. From Lemma \ref{lem:hstens} there exists a unitary transform  $U:H\otimes H \to \hs\left( H,H \right)$ which, for any simple tensor $x\otimes y$, we have $U(x\otimes y) = x\l<y,\cdot\r>$.

		First we will show that the rank is greater than or equal to $l$ by contradiction. Suppose that $g = \sum_{i=1}^{l'}x_i \otimes y_i = h$ with $l'<l$. Since $l'<l$ there must exist some $j$ such that $h_j \notin \spn\left( x_1,\ldots,x_{l'} \right)$. Let $z\perp x_1,\ldots,x_{l'}$ and $z \not\perp h_j$. Now we have
		\begin{eqnarray*}
			\l<z\otimes z, h\r> = \sum_{i=1}^m \l<z,h_i\r>^2 \ge \l<z,h_j\r>^2 > 0,
		\end{eqnarray*}
		but 
		\begin{eqnarray*}
			\l<z\otimes z, g\r> = \sum_{i=1}^{l'} \l<z,x_i\r> \l<z, y_i\r> = 0,
		\end{eqnarray*}
		a contradiction.

		For the other direction, observe that $U(h)$ is a compact Hermitian operator and thus admits an spectral decomposition (\cite{introhilb} Theorem 8.15). From this we have that $U(h) = \sum_{i=1}^m h_i \l<h_i,\cdot\r> =  \sum_{i=1}^\infty \lambda_i \l<\psi_i, \cdot\r> \psi_i$  with $\left( \psi_i \right)_{i=1}^\infty$ orthonormal and $\lambda_i\ge 0$ for all $i$ since $U(h)$ is PSD. Clearly the dimension of the span of $U\left( h \right) $ is less than or equal to $ l$ and thus this decomposition has exactly $l$ nonzero terms. From this we can let $U(h) = \sum_{i=1}^l \lambda_i \l<\psi_i, \cdot\r> \psi_i$ and applying $U^{-1}$ we have that $h=\sum_{i=1}^l \lambda_i \psi_i^{\otimes 2}$. From this it follows that the rank of $h$ is less than or equal to $l$ and we are done.
	\end{proof}
	\begin{proof}[Proof of Lemma \ref{lem:noidentdown}]
		The lemma is obvious when $n=n'$. Assume that $n'<n$. Let $A \in \sG^{\times n'}$ be arbitrary. We have that 
		\begin{eqnarray*}
			\sum_{i=1}^m a_i \gamma_i^{\times n}\left( A\times \Psi^{\times n-n'} \right) &=&  \sum_{j=1}^l b_j \pi_j^{\times n}\left( A\times \Psi^{\times n-n'} \right)\\
			\Rightarrow \sum_{i=1}^m a_i \gamma_i^{\times 'n}\left( A\right) \gamma_i^{\times n-n'}\left( \Psi^{\times n-n'} \right) &=&  \sum_{j=1}^l b_j \pi_j^{\times n'}\left( A\right) \pi_j^{\times n- n'}\left( \Psi^{\times n-n'} \right)\\
			\Rightarrow \sum_{i=1}^m a_i \gamma_i^{\times n'}\left( A\right) &=&  \sum_{j=1}^l b_j \pi_j^{\times n'}\left( A\right).
		\end{eqnarray*}
		Since $A$ was chosen arbitrarily we have that $\sum_{i=1}^m a_i \gamma_i^{\times n'} =  \sum_{j=1}^l b_j \pi_j^{\times n'}$.
	\end{proof}

	\begin{proof}[Proof of Lemma \ref{lem:himbed}]
		Let $\pi= \sum_{i=1}^n \gamma_i$. Because $\pi$ is $\sigma$-finite for all $i$ we can define $f_i = \frac{d\gamma_i}{d\pi}$, where the derivatives are Radon-Nikodym derivatives. Let $f_k$ be arbitrary. We will first show that $f_k\le1$ $\pi$-almost everywhere. Suppose there exists a non $\pi$-null set $A \in \sG$ such that $f_i(A) >1$. Then we would have
		\begin{eqnarray*}
			\gamma_k\left( A \right)
			&=& \int_A f_k d\pi	\\
		 &>& \int_A 1 d\pi\\
		 &=& \sum_{i=1}^n \gamma_i(A)\\
		 &\ge& \gamma_k(A)
		\end{eqnarray*}
		a contradiction. From this we have
		\begin{eqnarray*}
			\int f_k^2 d\pi
			&\le& \int 1 d\pi\\
		 &\le& \sum_{i=1}^n \gamma_i(\Psi)\\
		 &<& \infty.
		\end{eqnarray*}
		From our construction it is clear that $f_i \ge 0$ $\xi$-almost everywhere so we can assert $f_i \ge 0$ without issue.
	\end{proof}

	\begin{proof}[Proof of Lemma \ref{lem:radprod}]
		The fact that $f$ is non-negative and integrable implies that the map $S \mapsto \int_S f^{\times n}d\pi^{\times n}$ is a bounded measure on $\l(\Psi^{\times n}, \sG^{\times n}\r)$ (see \cite{folland99} Exercise 2.12). 

		Let $R= R_1 \times\cdots\times R_n$ be a rectangle in $\sG^{\times n}$. Let $\mathds{1}_S$ be the indicator function for a set $S$. Integrating over $R$ and using Tonelli's theorem we get
		\begin{eqnarray*}
			\int_R f^{\times n} d \pi^{\times n}
			&=& \int \mathds{1}_Rf^{\times n}d \pi^{\times n}\\
		 &=& \int \l(\prod_{i=1}^n \mathds{1}_{R_i}(x_i)\r)\l(\prod_{j=1}^n f(x_j)\r)d \pi^{\times n}\left( x_1,\ldots,x_n \right)\\
		 &=& \int\cdots\int \l(\prod_{i=1}^n \mathds{1}_{R_i}(x_i)\r)\l(\prod_{j=1}^n f(x_j)\r)d \pi(x_1)\cdots d\pi(x_n)\\
		 &=& \int\cdots\int \l(\prod_{i=1}^n \mathds{1}_{R_i}(x_i) f(x_i)\r)d \pi(x_1)\cdots d\pi(x_n)\\
		 &=&  \prod_{i=1}^n\l(\int \mathds{1}_{R_i}(x_i) f(x_i)d \pi(x_i)\r)\\
		 &=&  \prod_{i=1}^n\gamma(R_i)\\
		 &=&  \gamma^{\times n}(R).
		\end{eqnarray*}
		Any product probability measure is uniquely determined by its measure over the rectangles (this is a consequence of Lemma 1.17 in \cite{fomp} and the definition of product $\sigma$-algebra) therefore, for all $B\in \sG^{\times n}$,
		\begin{eqnarray*}
			\gamma^{\times n}\left( B \right) = \int_B f^{\times n} d\pi^{\times n}.
		\end{eqnarray*}
	\end{proof}
	\section{Spectral Algorithm for Linearly Independent Components}\label{appx:spectral}
	Let $p_1,\ldots, p_m \in L^{2}\left( \Omega,\sF, \xi \right)$ be linearly independent pdfs with distinct norms. Their associated mixture proportions are $w_1,\ldots,w_m$. With four samples per random group we will have access to the tensors
	\begin{eqnarray}\label{eqn:appbigtensor}
		\sum_{i=1}^m w_i p_i^{\otimes 4}
	\end{eqnarray}
	and
	\begin{eqnarray} \label{eqn:appsmalltensor}
		\sum_{i=1}^m w_i p_i^{\otimes 2}.
	\end{eqnarray}

	We can transform the tensor in (\ref{eqn:appsmalltensor}) to an operator
	\begin{eqnarray*}
		C &\triangleq&  \sum_{i=1}^m w_i p_i\l<p_i,\cdot\r>\\
    &=& \sum_{i=1}^m \sqrt{w_i} p_i\l<\sqrt{w_i}p_i,\cdot\r>.
	\end{eqnarray*}
	Letting $W = \sqrt{C^\dagger}$ we have that $W \sqrt{w_1} p_1,\ldots,W \sqrt{w_m} p_m$ are orthonormal. Applying $I\otimes W \otimes I \otimes W$ to the tensor in (\ref{eqn:appbigtensor}) we can construct the tensor
	\begin{eqnarray*}
		\sum_{i=1}^m w_i p_i \otimes Wp_i \otimes p_i \otimes Wp_i = \sum_{i=1}^m  p_i \otimes W\sqrt{w_i}p_i \otimes p_i \otimes W\sqrt{w_i}p_i.
	\end{eqnarray*}
	which can be transformed into the operator
	\begin{eqnarray} \label{eqn:appbigop}
		\sum_{i=1}^m  p_i \otimes W\sqrt{w_i}p_i \l< p_i \otimes W\sqrt{w_i}p_i,\cdot\r>.
	\end{eqnarray}

	Note that for $i\neq j$ we have
	\begin{eqnarray*}
		\l<p_i\otimes W \sqrt{w_i}p_i,p_j\otimes W \sqrt{w_j}p_j\r> = \l<p_i,p_j\r> \l<W \sqrt{w_i} p_i, W \sqrt{w_j}\r> = 0.
	\end{eqnarray*}
	We also have that, for all $i$
	\begin{eqnarray*}
		\l\|p_i\otimes W\sqrt{w_i} p_i\r\|
		&=& \sqrt{\l<p_i\otimes W\sqrt{w_i} p_i,p_i\otimes W\sqrt{w_i} p_i\r>}\\
  &=& \sqrt{\l<p_i , p_i\r>\l<W\sqrt{w_i}p_i, W\sqrt{w_i}p_i\r>}\\
  &=& \sqrt{\l<p_i , p_i\r>}\\
  &=& \l\|p_i\r\|
	\end{eqnarray*}
	and thus the tensors $p_1\otimes W\sqrt{w_1}p_1, \ldots,p_m\otimes W\sqrt{w_m}p_m$ have distinct norms. Because of this the spectral decomposition of the operator in (\ref{eqn:appbigop}) will yield the eigenvectors $p_1 \otimes W\sqrt{w_1}p_1, \ldots,p_m \otimes W\sqrt{w_m}p_m$. Then, using the techniques from Appendix \ref{appx:proofs}, we can recover the mixture components and mixture proportions.
	\bibliographystyle{plain}
	\bibliography{rvdm}

	\end{document}